\documentclass[10pt,journal,compsoc]{IEEEtran}

\usepackage{amsopn}
\usepackage{bm}
\usepackage{amsthm}
\usepackage{etex}
\usepackage{spconf,graphicx}
\usepackage{endnotes}
\usepackage{hyperref}
\usepackage{epsfig,psfrag}
\usepackage{pst-all}
\usepackage{amssymb,amsfonts,upref,cite,epsf,color,bm}
\usepackage{graphicx}
\usepackage{color}
\usepackage{csvsimple}
\usepackage{url}
\usepackage{amsmath}
\usepackage{graphicx}
\usepackage{calc}
\usepackage{booktabs}
\usepackage{tikz,stackengine}
\usepackage{pgfplots}
\usetikzlibrary{calc, shapes, fit, positioning}
\newtheorem{assumption}{Assumption}
\newcommand\defeq{:=}

\DeclareMathOperator*{\argmin}{arg\;min}

\newcommand\vect[1]{\mathbf #1}
\newcommand{\va}{\vect{a}}  
\newcommand{\vb}{\vect{b}}

\newcommand{\vm}{\vect{m}}

\newcommand{\vt}{\vect{t}}
\newcommand{\vu}{\vect{u}}  
\newcommand{\vv}{\vect{v}}  
\newcommand{\vw}{\vect{w}}
\newcommand{\vx}{\vect{x}}  
\newcommand{\vz}{\vect{z}}

\newcommand{\mD}{\mathbf{D}}
\newcommand{\mL}{\mathbf{L}}

\newcommand{\mS}{\mathbf{S}}

\newcommand{\trueweights}{\overline{\vw}}
\newcommand{\samplesize}{M}

\newcommand{\signalsize}{N}

\newcommand{\graphsigs}{\mathcal{W}}

\newcommand{\edges}{\mathcal{E}}
\newcommand{\cluster}{\mathcal{C}}
\newcommand{\nodes}{\mathcal{V}}
\newcommand{\graph}{\mathcal{G}}
\newcommand{\samplingset}{\mathcal{M}}
\newcommand{\trainingset}{\mathcal{M}}

\newcommand{\expect}{{\rm E}}

\newcommand{\partition}{\mathcal{P}}
\newcommand{\boundary}{\partial \partition}

\newcommand{\compbound}{\overline{\partial \partition}}

\newcommand{\FIMeignmin}{L}
\newcommand{\FIMeignmax}{U}
\newcommand{\prob}{{\rm P}}
\newcommand{\numnodes}{N}
\newcommand{\numedges}{E}

\newcommand{\gindex}[1][i]{^{(#1)}}
\newcommand{\gsignal}{\vw}

\newcommand{\nodeidx}{i}
\newcommand{\gweight}{A}
\newcommand{\weightmtx}{\mathbf{A}}
\newcommand{\sigdimens}{d}
\newcommand{\featuredim}{d}
\newcommand{\featurelen}{d}
\newcommand{\FIM}{\mathbf{F}}
\newcommand{\FIMentry}{F}

\newcommand{\gvariable}{\mathbf{v}}
\newcommand{\gvariablep}{\mathbf{v}'}
\newcommand{\gdual}{\mathbf{u}}
\newcommand{\hatgsignal}{\widehat{\gsignal}}
\newcommand{\incidence}{\vect{D}}

\usepackage{algorithm}
\usepackage{algpseudocode}
\floatname{algorithm}{Algorithm}
\algnewcommand\algorithmicinput{\textbf{Input:}}
\algnewcommand\INPUT{\item[\algorithmicinput]}
\algnewcommand\algorithmicoutput{\textbf{Output:}}
\algnewcommand\OUTPUT{\item[\algorithmicoutput]}
\newtheorem{theorem}{Theorem}

\newtheorem{lemma}[theorem]{Lemma}

\newtheorem{corollary}[theorem]{Corollary}

\title{Learning Networked Exponential Families with Network Lasso}
\name{\hspace*{-10mm} Alexander Jung$^{1}$
}

\address{\normalsize $^1$Department of Computer Science, Aalto University, Espoo, Finland; firstname.lastname(at)aalto.fi \\
}

\begin{document}

\maketitle

\begin{abstract}
We propose networked exponential families to jointly leverage the information in the topology 
as well as the attributes (features) of networked data points. Networked exponential families 
are a flexible probabilistic model for heterogeneous datasets with intrinsic network structure.
These models can be learnt efficiently using network Lasso which implicitly pools or clusters 
the data points according to the intrinsic network structure and the local likelihood. The resulting 
method can be formulated as a non-smooth convex optimization problem which we solve using a 
primal-dual splitting method. This primal-dual method is appealing for big data applications 
as it can be implemented as a highly scalable message passing algorithm. 
\end{abstract}



\section{Introduction}
\label{sec_intro}

The data generated in many important application domains have an intrinsic network structure. 
Such networked data arises in the study of social networks, text document collections and 
personalized medicine \cite{RTMBlei2009,NetMedNat2010,Zachary77}. Network science 
provides powerful tools for the analysis of such data based on its intrinsic network structure \cite{NewmannBook,BigDataNetworksBook}. 
The network structure of datasets is complemented by the information 
contained in attributes (such as features or labels) of individual data points \cite{RTMBlei2009}. 
         
Consolidating prior work on networked (generalized) linear models \cite{LocalizedLinReg2019,LevinaNetworkPred}, 
we propose networked exponential families as a flexible probabilistic model for heterogenous and noisy 
data with an intrinsic network structure. By coupling the (node-wise) local parameters of an exponential family \cite{GraphModExpFamVarInfWainJor}, 
we jointly capitalize on network structure and the information conveyed by the features and labels of data points. 

Networked exponential families are powerful statistical models for many important application domains 
such as personalized (high-precision) health-care \cite{Lengerich2018}, 
or natural language processing \cite{Blei2003,RTMBlei2009}. 
In contrast to \cite{RTMBlei2009}, which uses a probabilistic model for the network structure of text corpora, 
this paper assumes the network structure as fixed and known. 

To learn networked exponential families, this paper implements the network Lasso in order to 
simultaneously cluster and optimize a probabilistic model \cite{NetworkLasso}. The implementation of 
nLasso is based on a primal-dual method which results in scalable message passing over the underlying data network. 
In contrast, to state-of-the art graph clustering methods which only use network structure, nLasso in networked exponential 
families jointly capitalizes on network structure and the information provided by observed node attributes. 
Joint clustering and optimization has been considered in \cite{PhysRevESSL} for probabilistic models of the network structure. 
In contrast, this paper considers the network structure fixed and given and use a probabilistic model for 
the node attributes (features and labels).

The idea of borrowing inferential power across networked data has also been used for bandit models 
in sequential decision making problems \cite{Gentile2014,Li2016}. In particular, the clustering bandit 
model coupled individual linear bandit models for nodes (representing users) using a domain-specific 
notion of similarity such as ``friendship'' relations in a social network. 

The closest to this work is \cite{LevinaNetworkPred} which considers regression with network cohesion (RNC). 
The RNC model is a special case of networked exponential families. While RNC  
uses a shared weight vector and a local (varying) intercept term, this paper allows for arbitrarily varying 
weight vectors (see end of Sec.\ \ref{sec_problem_formuation}). 

Another main difference between \cite{LevinaNetworkPred} and our approach is the 
choice of regularizer for the networked model. While \cite{LevinaNetworkPred}, similar to most existing 
work on semi-supervised learning \cite{SemiSupervisedBook}, uses the graph Laplacian 
quadratic form as a smoothness measure, our approach controls the non-smooth total variation (TV) of the 
model parameters. TV-based regularization produces predictors which are piece-wise constant over 
well-connected subset of nodes. This behaviour is useful in image processing of natural images which are 
composed or homogenous segments whose boundaries result in sharp edges \cite{Goldfarb2009}.


Minimizing the Laplacian quadratic form is a smooth convex problem resulting in a linear 
system. In contrast, TV minimization is a non-smooth convex 
optimization problem which requires more advanced techniques such as proximal methods 
\cite{pock_chambolle,ProximalMethods} (see Sec.\ \ref{sec_NLasso_PrimDual}). The higher 
computational cost of TV minimization affords improved accuracy when learning from a small 
number of labeled data points (see \cite{NSZ09} and Sec.\ \ref{sec_two_cluster}).

In order to learn networked exponential families, this paper applies the network Lasso (nLasso). 
The nLasso has been proposed recently as a natural extension of the Lasso to networked data \cite{NetworkLasso,HastieWainwrightBook}. 
We show how the nLasso can be implemented efficiently using a primal dual splitting method for 
convex optimization. The resulting scalable learning method amounts to a message 
passing protocol over the data network structure. 

\textbf{Contribution.} The main contributions of this paper are: 
\begin{itemize}
\item The introduction of networked exponential families as a 
probabilistic model for networked data. 
\item Extending prior \cite{JungAISTATS2019,LocalizedLinReg2019}, a bound on the nLasso error for general networked exponential families is presented.
\item A scalable nLasso implementation using a primal-dual method for convex optimization. The proposed 
formulation generalizes the method in \cite{Ambos2018} (for logistic regression) to arbitrary exponential families. 
\item Verification of computational and statistical properties of the proposed 
method using numerical experiments. 
\end{itemize}

\textbf{Outline.} We introduce networked exponential families in Sec.\ \ref{sec_problem_formuation}. 
Sec.\ \ref{sec_some_examples} details how some 
recently proposed models for networked data are obtained as special cases of networked exponential families. 
In Sec.\ \ref{sec_NLasso}, we show how to learn a networked exponential family using an instance of the nLasso 
optimization problem. We present an analysis of the nLasso estimation error in Sec.\ \ref{equ_analysis_error}. 
Sec.\ \ref{sec_NLasso_PrimDual} presents the implementation of nLasso using a primal-dual method for convex 
optimization. The computational and statistical properties of nLasso in networked exponential families are illustrated 
in numerical experiments within Sec.\ \ref{sec_numexp}.

\textbf{Notation.} 
We denote the $\ell_2$-norm of a vector as $\|\vect{x} \| \defeq \sqrt{\vect{x}^T \vect{x}}$. 
The spectral norm of a matrix is $\| \mathbf{M} \| \defeq \sup_{\| \vx\| \leq 1} \| \mathbf{M} \vx\|$.  
The convex conjugate of a function $f$ is $f^*(\vect{y})\!\defeq\!\sup_{\vect{x}} (\vect{y}^T \vect{x}\!-\!f(\vect{x}))$. 
The vector $\mathbf{e}^{(j)} \in \mathbb{R}^{\featurelen}$ denotes the $j$th column of the identity matrix of size $\featurelen \times \featurelen$. 

\section{Networked Exponential Families}
\label{sec_problem_formuation} 

We consider networked data represented by an undirected weighted graph (the 
``empirical graph'') $\graph\!=\!(\nodes, \edges, \weightmtx)$. The nodes $i\!\in\!\nodes\!=\!\{1, \ldots, \numnodes\}$ 
represent individual data point (such as social network users).
Data points $i,j\!\in\!\nodes$ are connected by an undirected edge $e=\{i,j\}\!\in\!\edges$ with 
weight 
\begin{equation} 
\label{equ_def_edge_weight}
\gweight_e = \gweight_{ij} > 0
\end{equation} 
if they are considered similar (e.g., befriended users). 
We denote the edge set $\edges$ by $\{1, \ldots, \numedges\defeq|\edges|\} $. 
The neighbourhood of a node $i \in \nodes$ is $\mathcal{N}(i) \defeq \{ j : \{i,j\} \in \edges \}$.

In what follows, we assume the empirical graph $\graph$ fixed and known. The network structure 
might be induced by physical proximity (in time or space), physical connection (communication networks) 
or statistical dependency (probabilistic graphical models) \cite{GraphModExpFamVarInfWainJor}. 
The learning of network structure in a data-driven fashion \cite{CSGraphSelJournal,Dong2019} is beyond 
the scope of this paper.  
	
Beside network structure, datasets convey additional information via 
attributes $\vz\gindex\!\in\!\mathbb{R}^{\sigdimens}$ of data points $i\!\in\!\nodes$. 
We model the attributes $\vz\gindex$ of data points $i \in \nodes$ as independent random 
variables distributed according to (a member of) some exponential family \cite{GraphModExpFamVarInfWainJor}
\begin{equation} 
\label{equ_def_p_i}
p(\vz\gindex;\trueweights\gindex)\!\defeq\!b\gindex (\vz\gindex) \exp\big( (\trueweights\gindex)^{T}\vt\gindex(\vz\gindex)\!-\!\Phi\gindex(\trueweights\gindex) \big). 
\end{equation} 
The distribution \eqref{equ_def_p_i} is parametrized by the (unknown) weight vectors $\trueweights\gindex$. 
These weight vectors as fixed (deterministic) but unknown and the main focus of this paper is the accurate 
estimation of these weight vectors. 

It is convenient to collect weight vectors $\vw\gindex$ assigned to each n ode $i$ into a vector-valued graph 
signal $\vw: \nodes \rightarrow \mathbb{R}^{\featuredim}$ which maps a node $i$ to the function value $\vw\gindex$. 
The space of all such vector-valued graph signals is 
\begin{equation}
\label{equ_def_graph_sigs}
\mathcal{W} \defeq \{ \vw: \nodes \rightarrow \mathbb{R}^{\featuredim}: i \mapsto \vw\gindex \}. 
\end{equation} 
Similarly, we define the space of all vector-valued signals defined on the edges $\edges$ of 
the empirical graph as 
\begin{equation}
\label{equ_def_graph_sigs_edges}
\mathcal{D} \defeq \{ \vu: \edges \rightarrow \mathbb{R}^{\featuredim}:e \mapsto \vu^{(e)} \}. 
\end{equation} 

Strictly speaking, \eqref{equ_def_p_i} represents a probability density function relative to some underlying base measure $\nu$ 
defined on the value range of the sufficient statistic $\vt\gindex(\vz\gindex)$. Important examples of such a base 
measure are the counting measure for discrete-valued $\vt\gindex$ or the Lesbegue measure for continuous-valued $\vt\gindex$. 
The distribution defined by \eqref{equ_def_p_i} depends on $\vz\gindex$ only via the 
sufficient statistic $\vt\gindex(\vz\gindex)$. In what follows, we suppress the argument and write $\vt\gindex$ with 
the implicit understanding that it is a function of the random vector $\vz\gindex$. 

Several properties of the exponential family \eqref{equ_def_p_i} can be  
read off the log-partition or cumulant function  \cite{GraphModExpFamVarInfWainJor}
\begin{equation}
\label{equ_def_cummulant_function}
\Phi\gindex(\vw\gindex) \defeq {\rm log} \int_{\vt} b(\vt) \exp(-\vt^{T} \vw\gindex ) \nu(d \vt). 
\end{equation} 
The Fisher information matrix (FIM) $\FIM\gindex$ for \eqref{equ_def_p_i} is the Hessian  
\begin{equation} 
\label{equ_def_entries_Hessian} 
\FIM\gindex = \nabla^{2} \Phi\gindex(\vw)\mbox{, }\FIMentry\gindex_{m,n}(\vw) \defeq \frac{\partial^{2} \Phi\gindex(\vw)}{\partial w_{m} w_{n}}.
\end{equation} 
The conditioning of $\FIM\gindex$ crucially influences the statistical and computational 
properties of the model \eqref{equ_def_p_i} (see Sec.\ \ref{equ_analysis_error} and \ref{sec_NLasso_PrimDual} ). 


Within a networked exponential family, the node-wise models \eqref{equ_def_p_i} are coupled by 
requiring the weight vectors $\trueweights\gindex$ to be similar for well-connected data points. 
In particular, we require the weight vectors to have a small total variation (TV)
\vspace*{0mm}
\begin{align}
\label{equ_def_TV_norm}
\| \gsignal \|_{\rm TV} & \defeq \sum_{\{i,j\}\in \edges} \gweight_{ij} \| \gsignal\gindex[j] - \gsignal\gindex \|. 
\end{align}
Requiring the weight vectors $\gsignal\gindex$, for $i \in \nodes$, to have small TV forces 
weight vectors to be approximately constant over well connected subsets (clusters) of nodes. 
It will be convenient to define the TV for a subset $\mathcal{S}$ of edges: 
\begin{equation} 
\label{equ_def_TV_norm_subs}
\| \gsignal \|_{\rm \mathcal{S}} \defeq \sum_{\{i,j\}\in \mathcal{S}} \gweight_{ij} \| \gsignal\gindex[j] - \gsignal\gindex \|. 
\end{equation}

Let us finally compare networked exponential families, obtained as the combination of 
\eqref{equ_def_p_i} with a constraint on the TV \eqref{equ_def_TV_norm_subs} of weights 
$\trueweights$ in \eqref{equ_def_TV_norm_subs}, and the RNC model put forward in \cite{LevinaNetworkPred}. 
First, RNC considers the special case of distributions \eqref{equ_def_p_i} with $\vt\gindex = \big(\big(\vx\gindex \big)^{T}, 1 \big)^{T}$ 
and a partitioned weight vector $\trueweights=\big( {\bm \beta}^{T}, \alpha\gindex \big)^{T}$ with a 
shared weight vector ${\bm \beta}$ which is the same for all nodes $i\in \nodes$. The intercept $\alpha\gindex$ 
is allowed to vary over nodes. In contrast, we allow the entire weight vector $\trueweights$ to vary 
between different nodes. Moreover, while the RNC model uses the smooth Laplacian quadratic form of the 
intercepts $\alpha\gindex$, we use the non-smooth TV \eqref{equ_def_TV_norm} to measure how 
well the weight vectors conform with the network structure of the data. 

\vspace*{-4mm}	
\section{Some Examples} 
\label{sec_some_examples}
We now discuss important special cases of the model 
\eqref{equ_def_p_i}. 
\vspace*{-4mm}
\subsection{Networked Linear Regression} 
Consider a networked dataset whose data points $i \in \nodes$ are characterized by features 
$\vx\gindex \in \mathbb{R}^{\featuredim}$ and numeric labels $y\gindex \in \mathbb{R}$. 
Maybe the most basic (yet quite useful) model for the relation between features and labels 
is the linear model 
\begin{equation}
\label{equ_lin_model}
y\gindex = (\vx\gindex)^{T} \vw\gindex + \varepsilon\gindex, 
\end{equation} 
with Gaussian noise $\varepsilon\gindex \sim \mathcal{N}(0,\sigma^{2})$ of known variance $\sigma_{i}^{2}$ 
which can vary for different nodes $i \in \nodes$. The linear model \eqref{equ_lin_model} is parametrized by 
the weight vectors $\vw\gindex$ for each $i\in \nodes$. The weight vectors are coupled by requiring a small TV \eqref{equ_def_TV_norm} \cite{LocalizedLinReg2019}.

The model \eqref{equ_lin_model} is obtained as the special case of the exponential family \eqref{equ_def_p_i} 
for the scalar attributes $z\gindex \defeq y\gindex$ with $\mathbf{t}\gindex(z) = (z/\sigma_{i}^2) \vx\gindex$ 
and $\Phi\gindex(\vw) =   ( \vw^{T} \vx\gindex )^2/(2\sigma_{i}^{2})$.

In some applications it is difficult to obtain accurate label information, i.e., $y\gindex$ is not known for some data point 
$i \in \nodes$. One approach to handle such partially labeled data is to use some 
crude estimates $\hat{y}\gindex$ of the labels for unlabelled nodes. We can account for 
varying label accuracy using heterogeneous noise variables $\varepsilon\gindex$. In particular, 
we use a larger noise variance $\sigma_{i}^{2}$ for a node $i \in \nodes$ for which 
we only have an estimate $\hat{y}\gindex$.


\subsection{Networked Logistic Regression} 
\label{sec_netLogRg}

Consider networked data points $i \in \nodes$ each characterized by features 
$\vx\gindex \in \mathbb{R}^{\featuredim}$ and binary labels $y\gindex \in \{-1,1\}$. 
Logistic regression models the relation between features and labels via
\begin{equation}
\label{equ_log_reg_model}
p( y\gindex = 1; \vw\gindex) \defeq 1/(1+ \exp(- (\vw\gindex)^{T} \vx\gindex)). 
\end{equation} 
The distribution \eqref{equ_log_reg_model} is parametrized 
by the weight vector $\vw\gindex$ for each node $i \in \nodes$. It can be shown 
that \eqref{equ_log_reg_model} is the posterior distribution of label $y\gindex$ 
given the features $\vx\gindex$ if the features $\vx\gindex$ is a Gaussian random 
vector conditioned on $y\gindex$. 

Networked logistic regression requires the weight vectors in the 
node-wise models \eqref{equ_log_reg_model} to have 
a small TV \eqref{equ_def_TV_norm} \cite{Ambos2018}. 

We obtain the logistic regression model \eqref{equ_log_reg_model} as the special case 
of the exponential family \eqref{equ_def_p_i} for the scalar node attributes $z\gindex \defeq y\gindex$ with 
$\mathbf{t}\gindex(z) = \vx\gindex  z/2$ and 
$\Phi\gindex(\vw) = \log \big( \exp \big(   \vw^{T}  \vx\gindex/2 \big) + \exp \big( -\vw^{T}  \vx\gindex/2  \big) \big)$. 

\subsection{Networked LDA} 
Consider a networked dataset representing a collection of text documents (such as scientific articles). 
The LDA is a probabilistic model for the relative frequencies of words 
in a document \cite{GraphModExpFamVarInfWainJor,Blei2003}. Within LDA, each document is considered 
a blend of different topics. Each topic has a characteristic distribution of the words in the vocabulary. 

A simplified form of LDA represents each document $i \in \nodes$ containing $N$ ``words'' by two sequences 
of multinomial random variables $z\gindex_{w,1},\ldots,z\gindex_{w,N} \in \{1,\ldots,W\}$ and $z\gindex_{t,1},\ldots,z\gindex_{t,N} \in \{1,\ldots,T\}$ 
with $V$ being the size of the vocabulary defining elementary words and $T$ is the number of different topics. 
It can be shown that LDA is a special case of the exponential family \eqref{equ_def_p_i} with particular choices for 
$\vt(\cdot)$ and $\Phi\gindex(\cdot)$ (see \cite{GraphModExpFamVarInfWainJor,Blei2003}). 


\section{Network Lasso}
\label{sec_NLasso}

The goal of this paper is to develop a method for learning an accurate 
estimate $\widehat{\gsignal}\gindex$ for the true weights $\trueweights\gindex$ (see \eqref{equ_def_p_i}). 
The learning of the weight vectors $\gsignal\gindex$ is based on the availability of the 
nodes attributes $\vz\gindex$ for a small ``training set'' $\samplingset = \{i_{1},\ldots,i_{\samplesize} \} \subseteq \nodes$. 
A reasonable estimate for the weight vectors can be obtained from maximizing the likelihood 
of observing the attributes $\vz\gindex$: 
\begin{align}
\label{equ_def_lilekihood}
p\big( \{ \vz\gindex\}_{i \in \samplingset} \big) & =  \prod_{i\in \samplingset} p( \vz\gindex; \vw\gindex)  \nonumber \\ 
  & \hspace*{-18mm} \stackrel{\eqref{equ_def_p_i}}{=}  \prod_{i\in \samplingset} b\gindex (\vz\gindex) \exp\big( \big(\vt\gindex\big)^{T} \vw\gindex - \Phi\gindex(\vw\gindex) \big). 
\end{align} 
Maximizing \eqref{equ_def_emp_risk} is equivalent to minimizing 
\begin{align}
\label{equ_def_emp_risk}
\widehat{E}(\vw)\!\defeq\!(1/\samplesize) \sum_{i \in \samplingset}   -  \big(  \vt\gindex\big)^{T} \vw\gindex  + \Phi\gindex(\vw\gindex). 
\end{align}

Criterion \eqref{equ_def_emp_risk} is not enough to learn the weights $\gsignal\gindex$ 
for all $i\!\in\!\nodes$. Indeed, \eqref{equ_def_emp_risk} ignores weights $\widehat{\gsignal}\gindex$ 
at unobserved nodes $i \in \nodes \setminus \samplingset$. Therefore, we impose additional structure 
on the weight vectors. Any reasonable estimate $\widehat{\gsignal}\gindex$ should conform 
with the \emph{cluster structure} of the empirical graph $\graph$ \cite{NewmannBook}. 

Networked data is often organized as clusters (or communities) which 
are well-connected subset of nodes. Many supervised learning methods use a 
clustering assumption that nodes belonging to the same cluster represent similar data points. We implement 
this clustering assumption by requiring the parameter vectors $\vw\gindex$ in \eqref{equ_def_p_i} 
to have a small TV \eqref{equ_def_TV_norm}. 	
	
We are led to learning the weights $\widehat{\gsignal}$ for \eqref{equ_def_p_i} 
via the  \emph{regularized empirical risk minimization} (ERM)
\begin{align} \label{optProb}
\hatgsignal & \in \argmin_{\gsignal \in \graphsigs} \widehat{E}(\gsignal)  + \lambda \| \gsignal \|_{\rm TV}. 
\end{align}
The learning problem \eqref{optProb} is an instance of the generic nLasso problem \cite{NetworkLasso}.    
The parameter $\lambda$ in \eqref{optProb} allows to trade-off small TV  $\| \hatgsignal \|_{\rm TV}$ 
against small error $\widehat{E}(\hatgsignal)$ (cf. \eqref{equ_def_emp_risk}). Chosing $\lambda$ 
can be based on validation \cite{HastieWainwrightBook} or the error analysis in Sec.\ \ref{equ_analysis_error}. 
	
	
It will be convenient to reformulate \eqref{optProb} using 
the block-incidence matrix $\incidence \in \mathbb{R}^{(\sigdimens\numedges) \times (\sigdimens\numnodes)}$ as 
\begin{align}
\incidence_{e,i} = \begin{cases}
\gweight_{ij} \mathbf{I}_{\sigdimens} & e=\{i,j\}, i<j\\
-\gweight_{ij} \mathbf{I}_{\sigdimens}& e=\{i,j\}, i>j\\
\mathbf{0} & {\rm otherwise}.
\end{cases}
\label{equ_def_incident_matrix}
\end{align}
The $e$-th block of $\incidence \gsignal$ is $\gweight_{ij} (\gsignal\gindex - \gsignal\gindex[j])$ in \eqref{equ_def_TV_norm} 
and, in turn, 
\begin{equation}
\label{equ_TV_is_2_1_norm}
\| \vw \|_{\rm TV} = \| \mathbf{D} \vw \|_{2,1}  
\end{equation} 
with the norm $\| \mathbf{\vu} \|_{2,1} \defeq \sum_{e \in \edges} \| \vu^{(e)} \|_{2}$ defined on $\mathcal{D}$ (see \eqref{equ_def_graph_sigs_edges}). 
We can then reformulate the nLasso \eqref{optProb} as
\begin{align}\label{LNLprob}
 \hatgsignal \in \argmin_{\gsignal \in \mathcal{W}}   h(\gsignal) + g(\incidence \gsignal),
\end{align}
with $h(\gsignal)\!=\!\widehat{E}(\gsignal)$ and $g(\gdual)\!\defeq\!\lambda \| \gdual \|_{2,1}$. 

Related to the incidence matrix \eqref{equ_def_incident_matrix}, is the graph Laplacian  
\begin{equation}
\label{equ_def_graph_Laplacian_matrix}
\mathbf{L}= {\bf \Lambda} \otimes \mathbf{I}_{\featurelen} - \weightmtx \otimes \mathbf{I}_{\featurelen}, 
\end{equation}
with the weight matrix $\mathbf{A}$ (see \eqref{equ_def_edge_weight}) and the ``degree matrix'' 
\begin{equation}
\nonumber
{\bf \Lambda}\!=\!{\rm diag} \{ d_{1},\ldots,d_{\signalsize} \}\!\in\!\mathbb{R}^{\signalsize \times \signalsize} \mbox{, with } d_{i}\!\defeq\!\sum_{\{j,i\}\!\in\!\edges} A_{i,j}. 
\end{equation}
The eigenvalues of $\mathbf{L}$ reflect the connectivity of the graph $\graph$. 
A graph $\graph$ is connected if and only if $\lambda_{2} >0$, with $\lambda_{2}$ being the smallest non-zero eigenvalue. 
The spectral gap $\rho (\graph) \defeq \lambda_{2}$ provides a measure of the connectivity of the graph $\graph$.

The Laplacian matrix $\mL$ is closely related to the incidence matrix $\mD$ (see \eqref{equ_def_incident_matrix}). 
Both matrices have the same nullspace. Moreover, the spectrum of $\mD \mD^{T}$ 
coincides with the spectrum of $\mL$. The column blocks $\mS^{(j)} \in \mathbb{R}^{(\signalsize \featuredim) \times \featuredim}$ 
of the pseudo-inverse $\mD^{\dagger}\!=\!\big(\mS^{(1)},\ldots,\mS^{(|\edges|)} \big) \in \mathbb{R}^{(\signalsize\featuredim) \times (|\edges|\featuredim)}$ 
of $\mD$ satisfy 
\begin{equation}
\label{equ_bound_cols_pseudo_D}
\| \mS^{(j)} \|_{2,\infty} \leq \sqrt{2 \featurelen \max_{i,j} A_{i,j}} /\rho(\graph).
\end{equation}
This bound can be verified using the identity $\mD^{\dagger}\!=\!(\mD \mD^{T})^{\dagger} \mD^{T}$ and well-known 
vector norm inequalities (see, e.g., \cite{Horn85}). 

\section{Analysis of nLasso Estimation Error}
\label{equ_analysis_error}

We now characteize the statistical properties of nLasso by analysing the prediction error 
$\widetilde{\vw} = \widehat{\vw}- \trueweights$ incurred by a solution $\widehat{\vw}$ of the nLasso problem \eqref{optProb}. 
In order to analyze the error incurred by the nLasso \eqref{optProb}, we assume that the true weight vectors 
are clustered 
\begin{equation}
\label{equ_def_clustered_signal_model}
\trueweights\gindex \!=\! \sum_{\cluster \in \partition} \vv^{(\cluster)} \mathcal{I}_{\cluster}[i].
\end{equation} 
Here, $\vv^{(\cluster)} \in \mathbb{R}^{\sigdimens}$ is the value of the true weigh vector for all nodes in the 
cluster $\cluster$. We also used the indicator map  $\mathcal{I}_{\cluster}[i]\!=\!1$ for $i\!\in\!\cluster$ and 
$\mathcal{I}_{\cluster}[i]\!=\!0$ otherwise.

The model \eqref{equ_def_clustered_signal_model} involves a partitioning $\partition = \{\cluster_{1}, \ldots, \cluster_{|\partition|}\}$ of 
the nodes $\nodes$ into disjoint subsets $\cluster_{l}$. 
The model \eqref{equ_def_clustered_signal_model} 
is a special case of piece-wise polynomial signal model which allows the weight vectors to 
vary within each cluster \cite{ChenClustered2016}.

The model \eqref{equ_def_clustered_signal_model}, which is used in \cite{Gentile2014} 
for networked bandit models, is meant to provide predictors that approximate the observed 
data well. The analysis below indicates that nLasso methods are robust to model mismatch, 
i.e., the true underlying weight vectors in \eqref{equ_def_p_i} can be approximated well by 
\eqref{equ_def_clustered_signal_model}. 

\begin{assumption}
\label{asspt_weights_clustered}
Node attributes $\vz\gindex$ are distributed 
according to \eqref{equ_def_p_i} with weight vectors $\trueweights\gindex$ 
that are piece-wise constant over some partition $\partition\!=\!\{\cluster_{1},\ldots,\cluster_{|\partition|} \}$ (see \eqref{equ_def_clustered_signal_model}). 
We measure the clusteredness of the partition $\partition$ using the spectral gap 
\begin{equation}
\label{equ_partition_spectral_gap}
\rho_{\partition}  \defeq \min_{\cluster_{l} \in \partition} \rho(\cluster_{l}). 
\end{equation}
\end{assumption} 
We emphasize that the partition underlying the model \eqref{equ_def_clustered_signal_model} is only required 
for the analysis of the nLasso error. For the implementation of nLasso (see Sec.\ \ref{sec_NLasso_PrimDual}), 
we do not need any information about the partition $\partition$.

\begin{assumption}
\label{asspt_FIM_lower_bound}
The FIM $\FIM\gindex$ (see \eqref{equ_def_entries_Hessian}) is bounded as $\FIMeignmax \mathbf{I} \succeq \FIM\gindex \succeq \FIMeignmin \mathbf{I}$ 
for any weights $\trueweights$
with some constant $\FIMeignmin >1$. 
\end{assumption}

\begin{assumption} 
\label{def_NNSP}
There are constants $K,\FIMeignmin>1$ such that for any $\mathbf{z} \!\in\! \mathcal{W}$ (see \eqref{equ_def_graph_sigs}) 
which is piece-wise constant on partition $\partition$, 
\vspace*{-1mm}
\begin{equation}  
\label{equ_ineq_multcompcondition_condition}
\FIMeignmin \| \vz \|_{\partial \partition}  \leq K  \| \vz \|_{\samplingset}+  \| \mathbf{z} \|_{\compbound}.
\vspace*{-1mm}
\end{equation} 
\end{assumption} 
The main analytic result of this paper is an upper bound on the probability that the nLasso error 
exceeds a given threshold $\eta$. 
\begin{theorem} 
\label{thm_main_result}
Consider networked data $\graph$ and training set $\trainingset$ 
such that Asspt. \ref{asspt_weights_clustered}, \ref{asspt_FIM_lower_bound} and \ref{def_NNSP} 
are satisfied with (see \ref{equ_ineq_multcompcondition_condition})
\begin{equation} 
\label{equ_def_params_NCC}
\FIMeignmin\!>\!3 \mbox{, and } K \in (1,\FIMeignmin\!-\!2), 
\end{equation} 
and corresponding condition number $\kappa \defeq \frac{K\!+\!3}{L\!-\!3}>1$. 
Based on the observed noisy labels $y_{i}$, we estimate the underlying weight vectors $\trueweights$ 
using a solution  $\widehat{\vw}$ to the nLasso problem \eqref{optProb} with $\lambda \defeq \eta /(5 \kappa^2)$ using 
some pre-specified error level $\eta>0$. 
Then, 
	\begin{align}
	\label{lower_boung_prob_main_results}
	\prob \{  \| \hat{\mathbf{w}}-\bar{\mathbf{w}} \|_{\rm TV}  \!\geq\! \eta\} & \leq 2 |\partition| \max_{l=1,\ldots,|\partition|} \exp\bigg(\hspace*{-2mm}-\! \frac{|\cluster_{l}|\eta^2}{8\cdot 25 \featurelen \FIMeignmax\kappa^2} \bigg) \nonumber \\ 
	& \hspace*{-20mm}\!+\!  2|\edges| \exp \bigg(\!-\! \frac{\samplesize  \rho^2_{\partition} \eta^2} {64 \cdot 25 \FIMeignmax \featuredim \| \mathbf{A} \|^{2}_{\infty}\kappa^4}\bigg).
	\end{align}
\end{theorem}

The bound \eqref{lower_boung_prob_main_results} indicates that, for a prescribed accuracy level $\eta$, 
the training set size $M$ has to scale according to $\kappa^4 / \rho^2_{\partition}$. 
Thus, the sample size required by Alg.\ \ref{alg:PD} scales with the fourth power of the condition 
number $\kappa =  \frac{K\!+\!3}{\FIMeignmin\!-\!3}$ (see Asspt.\ \ref{def_NNSP}) and inversely with the spectral 
gap $\rho_{\partition}$ of the partitioning $\partition$. Thus, nLasso methods \eqref{optProb} (such as 
Alg.\ \ref{alg:PD}) require less training data if the condition number $\kappa$ is small and the 
spectral gap $\rho_{\partition}$ is large. This is reasonable, since having a small condition number 
$\kappa =  \frac{K\!+\!3}{\FIMeignmin\!-\!3}$ (see Asspt.\ \ref{def_NNSP}) typically requires the edges 
within clusters to have larger weights on average than the weights of the boundary edges. Moreover, it is 
reasonable that nLasso tends to be more accurate for a larger spectral gap $\rho_{\partition}$, which 
requires the nodes within each cluster $\cluster_{l}$ to be well connected. Indeed, an graph 
$\graph$ consisting of well-connected clusters $\cluster_{l}$ favours clustered graph signals (see \eqref{equ_def_clustered_signal_model}) 
as solutions of nLasso \eqref{optProb}.

\section{A Primal-Dual Method}
\label{sec_NLasso_PrimDual}
The  nLasso \eqref{LNLprob} is a convex optimization problem with a non-smooth objective function 
which rules out the use of gradient descent methods. However, the objective 
function is highly structured since it is the sum of a smooth convex function $h(\gsignal)$ and a 
non-smooth convex function $g(\incidence \gsignal)$,
which can be optimized efficiently when considered separately. This suggests to use some 
proximal method \cite{ProximalMethods} for solving \eqref{LNLprob}. 
 
One particular example of a proximal method is the alternating direction method of multipliers (ADMM) 
which has been considered in \cite{NetworkLasso}. However, we will choose another type of 
proximal method which is based on a dual problem to \eqref{LNLprob} \cite{PrecPockChambolle2011,pock_chambolle}. 
These primal-dual methods are attractive since their analysis provides natural choices 
for the algorithm parameters. In contrast, tuning the ADMM parameter is non-trivial \cite{Nishihara2015}. 
   
 \subsection{Primal-Dual Method}
The preconditioned primal-dual method \cite{PrecPockChambolle2011} launches from reformulating the 
problem \eqref{LNLprob} as a saddle-point problem 
\begin{align}
\label{equ_pd_prob}
\min_{\gsignal \in \mathbb{R}^{\sigdimens \numnodes}} \max_{\gdual \in \mathcal{D}} \gdual^T\incidence \gsignal  + h(\gsignal) - g^*(\gdual),
\end{align}
with the convex conjugate  $g^*$ of $g$ \cite{pock_chambolle}.

Any solution $(\hatgsignal, \widehat{\gdual})$ of \eqref{equ_pd_prob} is characterized by 
\cite{RockafellarBook} 
\begin{align}
-\incidence^T \widehat{\gdual}  \in \partial h(\hatgsignal) \mbox{, and }
\incidence \hatgsignal \in \partial g^*(\widehat{\gdual}).
\label{equ_pd_cond_1}
\end{align}
This condition is, in turn, equivalent to
\begin{align} 
\hatgsignal\!-\!\mathbf{T} \incidence^T \widehat{\gdual}&\!\in\!(\mathbf{I}_{\sigdimens\numnodes}\!+\!\mathbf{T} \partial h) (\hatgsignal)\mbox{, and } \nonumber \\
\widehat{\gdual}\!+\!\boldsymbol{\Sigma} \incidence \hatgsignal&\!\in\!(\mathbf{I}_{\sigdimens\numedges}\!+\!\boldsymbol{\Sigma} \partial g^*)(\widehat{\gdual}),
\label{equ_pd_cond_1_2}
\end{align}
with positive definite matrices $ \boldsymbol{\Sigma} \!\in\! \mathbb{R}^{\sigdimens \numedges \times \sigdimens\numedges}, \mathbf{T} \!\in\! \mathbb{R}^{\sigdimens\numnodes \times \sigdimens\numnodes}$. The matrices $\boldsymbol{\Sigma}, \mathbf{T}$ are design parameters whose choice will be detailed below. 
The condition \eqref{equ_pd_cond_1_2} lends naturally to the following coupled fixed point iterations \cite{PrecPockChambolle2011}
\begin{align}
\gsignal_{k+1} \!&=\! (\mathbf{I} \!+\! \mathbf{T} \partial h)^{-1} (\gsignal_{k} \!-\!\mathbf{T} \incidence^T \gdual_{k})  \label{equ_pd_upd_x} \\[3mm]
\gdual_{k+1} \!&=\! (\mathbf{I} \!+\! \boldsymbol{\Sigma} \partial g^*)^{-1} (\gdual_{k} \!+\! \boldsymbol{\Sigma} \incidence (2\gsignal_{k+1} \!-\! \gsignal_{k})).
\label{equ_pd_upd_y}
\end{align}

If the matrices $\boldsymbol{\Sigma}$ and $\mathbf{T}$ in \eqref{equ_pd_upd_x}, \eqref{equ_pd_upd_y} satisfy
\begin{align}
\|\boldsymbol{\Sigma}^{1/2} \incidence  \mathbf{T}^{1/2}\|^2 <1,
\label{equ_pre_cond}
\end{align}
the sequence $\gsignal_{k+1}$ (see \eqref{equ_pd_upd_x}, \eqref{equ_pd_upd_y}) converges to a solution of \eqref{optProb} \cite[Thm. 1]{PrecPockChambolle2011}. 
The condition \eqref{equ_pre_cond} is satisfied for 
\begin{equation}
\label{equ_def_diag_matrix_converge}
\hspace*{-4mm}\boldsymbol{\Sigma}\!\defeq\!{\rm diag} \{ (1/(2\gweight_{e})) \mathbf{I}\}_{e \in \edges}\mbox{, } \mathbf{T} \!\defeq\! {\rm diag} \{(\tau/d^{(i)})  \mathbf{I}\}_{i\in \nodes},
\end{equation} 
with $d^{(i)}\!=\!\sum_{j \neq i } \gweight_{ij} $ and some $\tau\!<\!1$ \cite[Lem. 2]{PrecPockChambolle2011}.

The update \eqref{equ_pd_upd_y} involves the resolvent operator 
\begin{align}
\label{equ_def_prox}
\hspace*{-3mm}(\mathbf{I} \!+\! \boldsymbol{\Sigma} \partial g^*)^{-1} (\gvariable) \!=\! \argmin_{\gvariablep \in \mathcal{D} } g^*(\gvariablep) \!+\! (1/2)\| \gvariablep \!-\! \gvariable\|^2_{\boldsymbol{\Sigma}^{-1}},
\end{align}
where  $\|\gvariable\|_{\boldsymbol{\Sigma}} \!\defeq\! \sqrt{\gvariable^T \Sigma \gvariable}$. 
The convex conjugate $g^*$ of $g$ (see \eqref{LNLprob}) can be decomposed as $g^*(\gvariable) = \sum_{e=1}^{\numedges} g_2^*(\gvariable^{(e)})$ 
with the convex conjugate $g_2^*$ of the scaled $\ell_2$-norm $\lambda \|.\|$. Moreover, since $\boldsymbol{\Sigma}$ is a 
block diagonal matrix, the $e$-th block of the resolvent operator $(\mathbf{I}_{\sigdimens\numedges} + \boldsymbol{\Sigma} \partial g^*)^{-1} (\gvariable)$ 
can be obtained by the Moreau decomposition as \cite[Sec. 6.5]{ProximalMethods}
\begin{align}
&\hspace*{-5mm} ((\mathbf{I}_{\sigdimens\numedges} + \boldsymbol{\Sigma} \partial g^*)^{-1} (\gvariable))^{(e)}  \nonumber \\[1mm]
& \hspace*{-1mm} \stackrel{\eqref{equ_def_prox}}{=} \argmin_{\gvariablep \in \mathbb{R}^{\sigdimens } } g_2^*(\gvariablep) \!+\! (1/(2\sigma^{(e)})) \| \gvariablep \!-\! \gvariable^{(e)}\|^2 \nonumber\\[3mm]
&= \gvariable^{(e)} \!-\! \sigma^{(e)} (\mathbf{I}_{\sigdimens} \!+\! (\lambda/ \sigma^{(e)})\partial \|.\|)^{-1} (\gvariable^{(e)}/\sigma^{(e)}) \nonumber\\[3mm]
&= \begin{cases}
\lambda \gvariable^{(e)}/ \|\gvariable^{(e)}\| & \text{if } \|\gvariable^{(e)}\|> \lambda\\
\gvariable^{(e)} & {\rm otherwise},
\end{cases} \nonumber
\end{align}
where $(a)_{+} \!=\! \max\{a, 0\}$ for $a \in \mathbb{R}$. 

The update \eqref{equ_pd_upd_x} involves the resolvent operator $(\mathbf{I} + \mathbf{T} \partial h)^{-1}$ of $h$ 
(see \eqref{equ_def_emp_risk} and \eqref{LNLprob}), which does not admit a simple closed-form solution in general. 
Using \eqref{equ_def_diag_matrix_converge}, 
the update \eqref{equ_pd_upd_x} decomposes into independent node-wise updates 
\begin{align}
\label{equ_pd_upd_x_exact}
& \hspace*{-3.8mm}\gsignal\gindex_{k+1}\hspace*{-1mm}\defeq\!  \begin{cases} 
& \argmin\limits_{\gsignal \in \mathbb{R}^{\sigdimens}} g^{(i)}(\vw) \mbox{ for } i \in \samplingset \\ 
&  \overline{\gsignal}\gindex \mbox{ for } i \in \nodes \setminus \samplingset \end{cases}
\end{align}
with $g^{(i)}(\vw)\!\defeq\!- \gsignal^T \vt\gindex\!+\!\Phi\gindex(\vw)\!+\!\tilde{\tau}\gindex\! \|\gsignal \!-\! \overline{\gsignal}\gindex\|^2$,  $\tilde{\tau}\gindex\defeq\samplesize/(2\tau\gindex)$ and 
\begin{equation} 
\label{equ_def_overline_gsignal}
\overline{\gsignal}\!\defeq\!\gsignal_{k}\!-\!\mathbf{T} \incidence^T \gdual_{k}.
\end{equation}  

It is important to note that the update \eqref{equ_pd_upd_x_exact}, for $i\!\in\!\samplingset$, amounts to a 
regularized maximum likelihood estimator for exponential families \cite[Eq. 3.38]{GraphModExpFamVarInfWainJor}. 
The regularization term $\!\tilde{\tau}\gindex\! \|\gsignal \!-\! \overline{\gsignal}\gindex\|^2$, which varies as iterations proceed, 
enforces $\gsignal\gindex_{k+1}$ to be close to $\overline{\gsignal}\gindex$. The vector $\overline{\gsignal}\gindex$ 
is a corrected version of the previous iterate $\gsignal\gindex_{k}$ (see \eqref{equ_def_overline_gsignal}).

In general, there is no closed-form solution for the update \eqref{equ_pd_upd_x_exact}.  
However, 
the update \eqref{equ_pd_upd_x_exact} is 
a smooth convex optimization problem that can be solved efficiently using iterative methods 
such as L-BGFS \cite{Mokhtari2015}. We detail a computationally cheap iterative method for 
approximately solving \eqref{equ_pd_upd_x_exact} in Sec.\ \ref{sec_fixed_point_iter}.

Let us denote the approximate solution to \eqref{equ_pd_upd_x_exact} by $\widehat{\vw}\gindex_{k+1}$ 
and assume that it is sufficiently accurate such that  
\begin{align}
e_k = \|\widehat{\gsignal}_{k+1}\gindex - {\gsignal}_{k+1}\gindex \| \leq 1/k^2.
\label{equ_pd_upd_x_err}
\end{align}
Thus, we require the approximation quality (for approximating the update \eqref{equ_pd_upd_x_exact}) to increase with the iteration number $k$. 
According to \cite[Thm. 3.2]{Condat2013}, the error bound \eqref{equ_pd_upd_x_err} ensures the sequences 
obtained by \eqref{equ_pd_upd_x} and \eqref{equ_pd_upd_y} when replacing the exact update \eqref{equ_pd_upd_x_exact} 
with the approximation $\widehat{\gsignal}_{k+1}$ still converge to a saddle-point of \eqref{equ_pd_prob} and, in turn, 
a solution of the nLasso problem \eqref{LNLprob}.

\begin{algorithm}[]
\caption{Primal-Dual nLasso}\label{alg:PD}
\begin{algorithmic}[1]
\renewcommand{\algorithmicrequire}{\textbf{Input:}}
\renewcommand{\algorithmicensure}{\textbf{Output:}}
\Require   $\graph = (\nodes, \edges, \mathbf{A})$, $\{\vz\gindex\}_{\nodeidx \in \samplingset}$, $\samplingset$, 
$\lambda$, $\incidence$
\Statex\hspace{-6mm}{\bf Init:} set $\boldsymbol{\Sigma},\mathbf{T}$ via \eqref{equ_def_diag_matrix_converge}, 
$k\!\defeq\!0$, $\widehat{\gsignal}_0\!\defeq\!0$, $\widehat{\gdual}_0\!\defeq\!0$
		\Repeat
			\State $\widehat{\gsignal}_{k+1} \defeq \widehat{\gsignal}_{k} - \mathbf{T} \incidence^T \widehat{\gdual}_{k}$
			\For{ each observed node $\nodeidx \in \samplingset$}
				\State 
				\vspace{-4mm}\begin{align}
				\mbox{ compute }\widehat{\gsignal}_{k+1}\gindex  \mbox{ by (approximately) solving } \eqref{equ_pd_upd_x_exact} \nonumber
				\end{align}
			\EndFor	
			\State  $\overline{\gdual} \defeq {\gdual}_k + \boldsymbol{\Sigma} \incidence (2\widehat{\gsignal}_{k+1}-\widehat{\gsignal}_{k})$ 

			\State  $\widehat{\gdual}_{k+1}^{(e)} = \overline{\gdual}^{(e)} - \bigg(1- \frac{\lambda}{\|\overline{\gdual}^{(e)}\|}\bigg)_{+} \overline{\gdual}^{(e)}$ for $e \in \edges$

		\State  $k\!\defeq\!k\!+\!1$
		\Until stopping criterion is satisfied 
		\Ensure $(\hatgsignal_{k},\widehat{\gdual}_{k})$.
	\end{algorithmic}
\end{algorithm}

The primal-dual implementation of nLasso in Alg.\ \ref{alg:PD} requires only the empirical graph along with 
the observed node attributes $\vz\gindex$, for $i \in \samplingset$, as input. As already mentioned above, 
Alg.\ \ref{alg:PD} does not require any specification of a partition of the empirical graph. Moreover, in contrast 
to the ADMM implementation of nLasso (see \cite[Alg.\ 1]{NetworkLasso}), the proposed Alg.\ \ref{alg:PD} does 
not involve unspecified tuning parameters. 
	
\subsection{Computational Complexity} 
It can be shown that Alg.\ \ref{alg:PD} can be implemented as message passing 
over the empirical graph $\graph$ (see \cite{Ambos2018}). During each iteration, messages 
are passed over each edge $\{i,j\} \in \edges$ in the empirical graph. The computation 
of a single message requires a constant amount of computation. The precise amount 
of computation required for a single message depends on the particular instance of the update \eqref{equ_pd_upd_x_exact}. 

For a fixed number of iterations used for Alg.\ \ref{alg:PD}, its complexity scales linearly with the 
number of edges $\edges$. For bounded degree graphs, such as grid or chain graphs, this 
implies a linear scaling of complexity with number of data points. 

However, the overall complexity for Alg.\ \ref{alg:PD} depends crucially on the number of iterations 
required to achieve accurate learning. A worst-case analysis shows that, for exact updates in \eqref{equ_pd_upd_x_exact}, 
the number of iterations scales inversely with the required estimation accuracy \cite{pock_chambolle}. Moreover, this 
convergence speed cannot be improved for chain graphs \cite{ComplexitySLP2018}. 

\vspace*{-2mm}
\subsection{Approximate Primal Update} 
\label{sec_fixed_point_iter}

We now detail a simple iterative method for computing an approximate solution $\widehat{\vw}\gindex_{k+1}$ 
to the primal update \eqref{equ_pd_upd_x_exact}. A solution $\widehat{\vw}$ of \eqref{equ_pd_upd_x_exact} is 
characterized by the zero gradient condition \cite{BoydConvexBook}
\begin{equation}
\label{equ_zero_gradient}
\nabla f(\widehat{\vw}) = \mathbf{0}
\end{equation} 
with $f(\vw) \defeq -\gsignal^T \vz\gindex\!+\!\Phi\gindex(\vw)\!+\!\tilde{\tau}\gindex\! \|\gsignal \!-\! \overline{\gsignal}\gindex\|^2$. 
Applying basic calculus to \eqref{equ_zero_gradient}, 
\vspace*{-2mm} 
\begin{equation} 
\label{equ_chara_fixed_point}
\vw\gindex =  \overline{\gsignal}\gindex + (\tau\gindex/\samplesize) \big( \vz\gindex - \nabla \Phi\gindex(\vw\gindex) \big). 
\end{equation} 
The necessary and sufficient condition \eqref{equ_chara_fixed_point} (for $\vw\gindex$ to 
solve \eqref{equ_pd_upd_x_exact}) is a fixed point equation $\vw\gindex\!=\!\mathcal{T}(\vw\gindex)$ with 
\begin{equation}
\label{equ_def_map_fixed_point}
\mathcal{T}: \mathbb{R}^{\featuredim}\!\rightarrow\!\mathbb{R}^{\featuredim}: \vw\!\mapsto\!\overline{\gsignal}\gindex\!+\!(\tau\gindex/\samplesize) \big( \vz\gindex - \nabla \Phi\gindex(\vw) \big). 
\end{equation}

By the mean-value theorem \cite[Thm. 9.19.]{RudinBookPrinciplesMatheAnalysis}, the map 
$\mathcal{T}$ is Lipschitz with constant $(\tau\gindex/\samplesize) \| \mathbf{F}(\vw)  \|$ where 
$\FIM\gindex$ is the FIM \eqref{equ_def_entries_Hessian}. 
Thus, if we choose $\tau\gindex$ such that  
\begin{equation} 
\label{equ_condition_condition}
R \defeq (\tau\gindex/\samplesize) \| \mathbf{F}(\vw)  \| < 1, 
\end{equation} 
the map $\mathcal{T}$ in \eqref{equ_def_map_fixed_point} is a contraction and the fixed-point iteration
\begin{align}
\label{equ_update_tilde_vw}
\hspace*{-2mm}\widetilde{\vw}^{(r\!+\!1)} & \!=\! \mathcal{T} \widetilde{\vw}^{(r)} \!\stackrel{\eqref{equ_def_map_fixed_point}}{=}\! 
 \overline{\gsignal}\gindex\!+\!(\tau\gindex\!/\!\samplesize) \big( \vz\gindex\!-\!\nabla \Phi\gindex(\widetilde{\vw}^{(r)}) \big)
\end{align}
will converge to a solution of \eqref{equ_pd_upd_x_exact}. 

Moreover, if \eqref{equ_condition_condition} is satisfied, we can bound the deviation 
between the iterate $\vw^{(r)}$ and the (unique) solution ${\vw}\gindex_{k+1}$ of \eqref{equ_condition_condition} 
as (see \cite[Proof of Thm. 9.23]{RudinBookPrinciplesMatheAnalysis})
\begin{equation}
\| \widetilde{\vw}^{(r)} - \vw\gindex \| \leq (R^{r}/(1\!-\!R)) \| \widetilde{\vw}^{(1)} - \widetilde{\vw}^{(0)} \|. 
\end{equation} 
Thus, if we use the approximation $\widehat{\vw}\gindex_{k+1} \defeq \widetilde{\vw}^{(r)}$ for the 
update \eqref{equ_pd_upd_x_exact}, we can ensure \eqref{equ_pd_upd_x_err} by iterating 
\eqref{equ_update_tilde_vw} for at least 
\begin{equation}
r \geq \log \big[ (1\!-\!R)\| \widetilde{\vw}^{(1)}\!-\!\widetilde{\vw}^{(0)} \| /k^2 \big] / \log R. 
\end{equation}     

Note that computing the iterates \eqref{equ_update_tilde_vw} requires the evaluation of 
the gradient $\nabla \Phi\gindex(\widetilde{\vw}^{(r)})$ of the log partition function $\Phi\gindex(\vw)$. 
According to \cite[Prop.\ 3.1.]{GraphModExpFamVarInfWainJor},
\vspace*{-2mm}
\begin{equation} 
\label{equ_nabla_expect}
\nabla \Phi\gindex(\vw) = {\rm E} \{ \vt(\vz\gindex) \} \mbox{ with } \vz\gindex \sim p(\vz;\vw). 
\vspace*{-2mm}
\end{equation} 
In general, the expectations \eqref{equ_nabla_expect} cannot be 
computed exactly in closed-form. A notable exception are exponential families $p(\vz;\vw)$ 
obtained from a probabilistic graphical model defined on a triangulated graph such as a tree. 
In this case it is possible to compute \eqref{equ_nabla_expect} in closed-form (see \cite[Sec. 2.5.2]{GraphModExpFamVarInfWainJor}). 
Another special case of \eqref{equ_def_p_i} for which \eqref{equ_nabla_expect} can be evaluated 
in closed-form is linear and logistic regression (see Sec.\ \ref{sec_some_examples}).

\vspace*{-1mm}
\subsection{Partially Observed Models} 
\vspace*{-1mm}

The learning Algorithm \ref{alg:PD} can be adapted easily to cope with partially observed exponential families \cite{GraphModExpFamVarInfWainJor}. 
In particular, for the networked LDA described in Sec.\ \ref{sec_some_examples}, we typically have access only to 
the word variables $z\gindex_{w,1},\ldots,z\gindex_{w,N}$ of some documents $i \in \samplingset \subseteq \nodes$. 
However, for (approximately) computing the update step \eqref{equ_pd_upd_x_exact} we would also need the 
values of the topic variables $z\gindex_{t,1},\ldots,z\gindex_{t,N}$ but those are not observed since they are latent (hidden) variables. 
In this case we can approximate \eqref{equ_pd_upd_x_exact} by some ``Expectation-Maximization'' (EM) principle (see \cite[Sec. 6.2]{GraphModExpFamVarInfWainJor}). 
An alternative to EM methods, based on the method of moments, for learning (latent variable) topic models 
has been studied in a recent line of work \cite{AroraTopModels2016}. 

\section{Numerical Experiments} 
\label{sec_numexp}

We report on the numerical results obtained by applying particular instances of Alg.\ \ref{alg:PD} 
to different datasets. The source code to reproduce these experiments can be found at \url{https://github.com/alexjungaalto/nLassoExpFamPDSimulations}. 

\vspace*{-1mm}
\subsection{Two-Cluster Dataset} 
\label{sec_two_cluster}
\vspace*{-1mm}

This experiment constructs an empirical graph $\graph$ by sparsely connecting two random graphs $\cluster_{1}$ 
and $\cluster_{2}$, each of size $\signalsize/2\!=\!40$ and with average degree $10$. 
The nodes of $\graph$ are assigned feature vectors $\vx^{(\nodeidx)} \in \mathbb{R}^{2}$ 
obtained by i.i.d.\ random vectors uniformly distributed on the unit sphere $\{ \vx \in \mathbb{R}^{2}: \| \vx \|=1\}$. 
The labels $y^{(\nodeidx)}$ of the nodes $i\in \nodes$ are generated according to the linear model \eqref{equ_lin_model} 
with zero noise $\varepsilon^{(\nodeidx)}=0$ and piecewise constant weight vectors $\mathbf{w}^{(\nodeidx)} = \mathbf{a}$ for $i \in \cluster_{1}$ and 
$\mathbf{w}^{(\nodeidx)} = \mathbf{b}$ for $i \in \cluster_{2}$ 
with some two (different) fixed vectors $\mathbf{a}, \mathbf{b} \in \mathbb{R}^{2}$. 
We assume that the labels $y^{(\nodeidx)}$ are known for the nodes in 
a small training set $\samplingset$ which includes three data points from each cluster, i.e., 
$|\samplingset \cap \cluster_{1}|= |\samplingset \cap \cluster_{2}|=3$. 

As shown in \cite{WhenIsNLASSO}, the validity of Asspt.\ \ref{equ_ineq_multcompcondition_condition}, depends on the connectivity of the 
cluster nodes with the boundary edges $\partial \defeq \{ \{i,j\} \in \edges: i \in \cluster_{1}, j \in \cluster_{2} \}$ which 
connect nodes in different clusters. 
In order to quantify the connectivity of the labeled nodes $\samplingset$ with the cluster boundary, we compute, 
for each cluster $\cluster_{l}$, the normalized flow value $\rho^{(l)}$ from one particular in each cluster $\cluster_{l}$ 
and the cluster boundary $\partial$. We normalize this flow by the boundary size $|\partial|$.

In Fig.\ \ref{fig_NMSEconnect}, we depict the normalized mean squared error (NMSE) $\varepsilon\!\defeq\!\| \overline{\vw}\!-\!\widehat{\vw} \|^{2}_{2} / \| \overline{\vw} \|^{2}_{2}$ 
incurred by Alg.\ \ref{alg:PD} (averaged over $10$ i.i.d.\ simulation runs) for varying connectivity, as 
measured by the empirical average $\bar{\rho}$ of $\rho^{(1)}$ and $\rho^{(2)}$ (having same distribution). 
According to Fig.\ \ref{fig_NMSEconnect} there are two regimes of levels of connectivity. For  
connectivity $\bar{\rho}\!>\!\sqrt{2}$, Alg.\ \ref{alg:PD} is able to learn piece-wise constant weights $\mathbf{w}^{(i)}$. 


\vspace*{0mm}
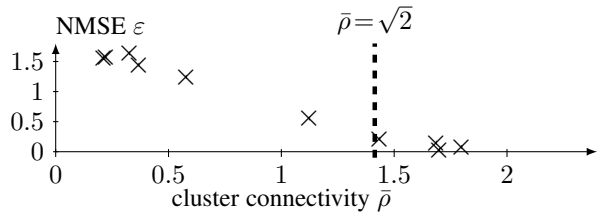
\begin{figure}[htbp]
\begin{center}
\begin{tikzpicture}
 \tikzset{x=3cm,y=0.8cm,every path/.style={>=latex},node style/.style={circle,draw}}
    \csvreader[ head to column names,%
                late after head=\xdef\aold{\a}\xdef\bold{\b},,%
                after line=\xdef\aold{\a}\xdef\bold{\b}]%
              {MSEoverBoundary_18-Sep-2019.csv}{}    
                {\draw [line width=0.5mm] (\a, \b) -- (\a,\b) node {\large $\times$};
    }
          \draw[->] (0,0) -- (2.4,0);
      \node [] at (1,-0.8) {\centering cluster connectivity $\bar{\rho}$};
      \draw[->] (0,0) -- (0,1.8);
      \node [anchor=south] at (0.2,1.8) {NMSE $\varepsilon$};
            \foreach \label/\labelval in {0/$0$,0.5/$0.5$,1/$1$,1.5/$1.5$}
        { 
          \draw (1pt,\label) -- (-1pt,\label) node[left] {\labelval};
        }
        \foreach \label/\labelval in {0/$0$,0.5/$0.5$,1/$1$,1.5/$1.5$,2/$2$}
        { 
          \draw (\label,1pt) -- (\label,-2pt) node[below] {\labelval};
        }
        
        \draw[dashed, line width=.6mm] (1.414,-0.1) -- (1.414,1.8) ;
        \node[anchor=south] at (1.414,1.8) {$\bar{\rho}\!=\!\sqrt{2}$} ; 
\end{tikzpicture}
\end{center}
\vspace*{-6mm}
  \caption{nLasso error for networked linear regression. } 
  \label{fig_NMSEconnect}
\end{figure}

To compare the effect of using TV \eqref{equ_def_TV_norm} in Alg.\ \ref{alg:PD} instead of the 
graph Laplacian quadratic form (see \cite{LevinaNetworkPred}) as network regularizer, 
a networked signal in noise model $y\gindex = w\gindex + \varepsilon\gindex$ is considered. 
The noise $\varepsilon\gindex$ is i.i.d.\ with zero mean and known variance $\sigma^{2}$. The signal weights are 
piece-wise constant with $\bar{w}\gindex = 1$ for $i\in \cluster_{1}$ and $\bar{w}\gindex = -1$ for $i \in \cluster_{2}$. 
Labels $y\gindex$ are observed only for nodes $\samplingset = \{1,2,3,\signalsize-2,\signalsize-1,\signalsize\}$. Alg.\ 
\ref{alg:PD} is used to learn weights $\hat{w}\gindex$ using a fixed number of $1000$ iterations and $\lambda =10$. 
The RNC estimator reduces to to one matrix inversion (see \cite[Eq. 2.4]{LevinaNetworkPred}) and is computed for 
the choices $\lambda \in \{1/100,1,100\}$ of the RNC regularization parameter. The resulting estimates $\hat{w}\gindex$ 
are shown in Fig.\ \ref{fig_nLassoRNC}. 

\vspace*{0mm}
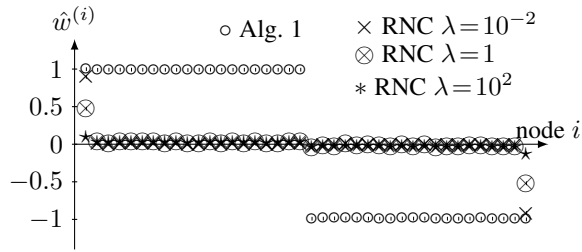
\begin{figure}[htbp]
\begin{center}
\begin{tikzpicture}
 \tikzset{x=3cm,y=1cm,every path/.style={>=latex},node style/.style={circle,draw}}
    \csvreader[ head to column names,%
                 late after head=\xdef\aold{\a}\xdef\bold{\b}\xdef\cold{\c}\xdef\dold{\d}\xdef\eold{\e},%
                after line=\xdef\aold{\a}\xdef\bold{\b}\xdef\cold{\c}\xdef\dold{\d}\xdef\dole{\e}
                
                ]%
              {nLassovsRNC_25-Sep-2019.csv}{}    
                {\draw [line width=0.5mm] (\a/20, \b) -- (\a/20,\b) node { $\circ$};
                \draw [line width=0.5mm] (\a/20, \c) -- (\a/20,\c) node { $\times$};
                    \draw [line width=0.5mm] (\a/20, \d) -- (\a/20,\d) node { $\otimes$};
                     \draw [line width=0.5mm] (\a/20, \e) -- (\a/20,\e) node { $\star$}; 
    }
     \draw[->] (0,0) -- (2.1,0);
     \node [] at (2.1,0.2) {\centering node $i$};
     \draw[->] (0,-1.4) -- (0,1.4);
      \node [anchor=south] at (0,1.4) { $\hat{w}\gindex$};
            \foreach \label/\labelval in {-1/$-1$,-0.5/$-0.5$,0/$0$,0.5/$0.5$,1/$1$}
        { 
          \draw (1pt,\label) -- (-1pt,\label) node[left] {\labelval};
        }
        
        \node[anchor=west] at (0.6,1.5) {$\circ$ Alg.\ \ref{alg:PD}} ; 
           \node[anchor=west] at (1.2,1.6) {$\times$ RNC $\lambda\!=\!10^{-2}$} ; 
              \node[anchor=west] at (1.2,1.2) {$\otimes$ RNC $\lambda\!=\!1$} ; 
                     \node[anchor=west] at (1.2,0.8) {$\ast$ RNC $\lambda\!=\!10^2$} ; 
\end{tikzpicture}
\end{center}
\vspace*{-6mm}
  \caption{Weights learnt by Alg.\ \ref{alg:PD} and RNC \cite{LevinaNetworkPred}. } 
  \label{fig_nLassoRNC}
\end{figure}
According to Fig.\ \ref{fig_nLassoRNC}, is able to accurately learn the piece-wise constant 
weights $\bar{w}\gindex$ from only two labels $y\gindex$, for $i \in \{1,\signalsize\}$. In contrast, 
RNC fails to leverage the network structure in order to learn the weights from a small number of 
labels.

\subsection{Weather Data} 


In this experiment, we consider networked data obtained from the Finnish meteorological institute. 
The empirical graph $\graph$ of this data represents Finnish weather stations.  
which are initially connected by an edge to their $K=3$ nearest neighbors. 
The feature vector $\vx^{(i)}\!\in\!\mathbb{R}^{3}$ of node $i\!\in\!\nodes$ contains the local (daily mean) 
temperature for the preceding three days. The label $y^{(i)} \in \mathbb{R}$ is the current day-average temperature. 

Alg.\ \ref{alg:PD} is used to learn the weight vectors $\vw^{(i)}$ for a localized linear model \eqref{equ_lin_model}. 
For the sake of illustration we focus on the weather stations in the capital region around Helsinki. 
These stations are represented by nodes $\cluster\!=\!\{23,18,22,15,12,13,9,7,5\}$ and we assume that labels $y^{(i)}$ are available for all 
nodes outside $\cluster$ and for the nodes $i\!\in\! \{12,13,15\}\!\subseteq\!\cluster$. Thus, for more than half of the nodes in $\cluster$ we do 
not know the labels $y^{(i)}$ but predict them via $\hat{y} = \big(\widehat{\vw}^{(i)}\big)^{T} \vx^{(i)}$ with the weight vectors $\widehat{\vw}^{(i)}$ 
obtained from Alg.\ \ref{alg:PD} (using $\lambda\!=1/7$ and a fixed number of $10^{4}$ iterations). The normalized 
average squared prediction error is  $\approx 10^{-1}$ and only slightly larger than the prediction error incurred by fitting 
a single linear model to the cluster $\cluster$.

\vspace*{-1mm}
\subsection{Image Segmentation}
\vspace*{-1mm}

This experiment revolves around using Alg.\ \ref{alg:PD} for image segmentation \cite{Grady2006,Rother2004}. 
Nodes $i\!\in\!\nodes$ are image pixels at coordinates $\big(p^{(i)},q^{(i)}\big)\!\in\!\{1,\ldots,P \}\!\times\!\{1,\ldots,Q\}$ (see Fig.\ \ref{fig:awesome_image2}). 

Different nodes $i,j$ are connected by an edge $\{i,j\} \in \edges$ if $p^{(i)} -p^{(j)} = 1$ or $q^{(i)} -q^{(j)} = 1$. 
We assign all edges $\{i,j\} \in \edges$ the same weight $W_{i,j} =1$. Pixels $i \in \nodes$ are characterized by 
feature vectors $\vx\gindex$ obtained by normalizing (zero mean and unit variance) the red, green and blue 
components of each pixel. 

We then constructed a training set $\samplingset$ of labeled data points by combining a background set 
$\mathcal{B}\subseteq \nodes$ ($y\gindex=0$) and a foreground set $\mathcal{F} \subseteq \nodes$ ($y\gindex=1$). 
These sets are determined based on the normalized redness $r\gindex \defeq x^{(i)}_{1}/\max_{j \in \nodes}  x^{(j)}_{1}$, 
\begin{equation}
\mathcal{B}\!\defeq\!\{i\!\in\!\nodes\!:\!r\gindex\!<\!1/2 \} \mbox{, and } \mathcal{F}\!\defeq\!\{i\!\in\!\nodes\!:\!r\gindex\!>\!9/10 \}. 
\end{equation}

We apply Alg.\ \ref{alg:PD}, with $\lambda\!=\!100$ and fixed number of $10$ iterations, 
to learn the weights $\vw\gindex$ for a networked logistic regression model (see Sec.\ \ref{sec_netLogRg}). 
For the update \eqref{equ_pd_upd_x_exact} in Alg.\ \ref{alg:PD} we used a single Newton step. 
The resulting predictions $\big(\widehat{\vw}\gindex\big)^{T} \vx\gindex$ are shown on the right of Fig.\ \ref{fig:awesome_image2}. 
The middle of Fig.\ \ref{fig:awesome_image2} depicts the hard segmentation obtained by the ``GrabCut'' 
method \cite{Rother2004}. Using {\rm MATLAB} version 19 on a standard laptop, Alg.\ \ref{alg:PD} is almost ten times faster than {\rm GrabCut}.

\begin{figure}[!htb]
\minipage{0.32\columnwidth}
  \includegraphics[width=3cm,height=2cm]{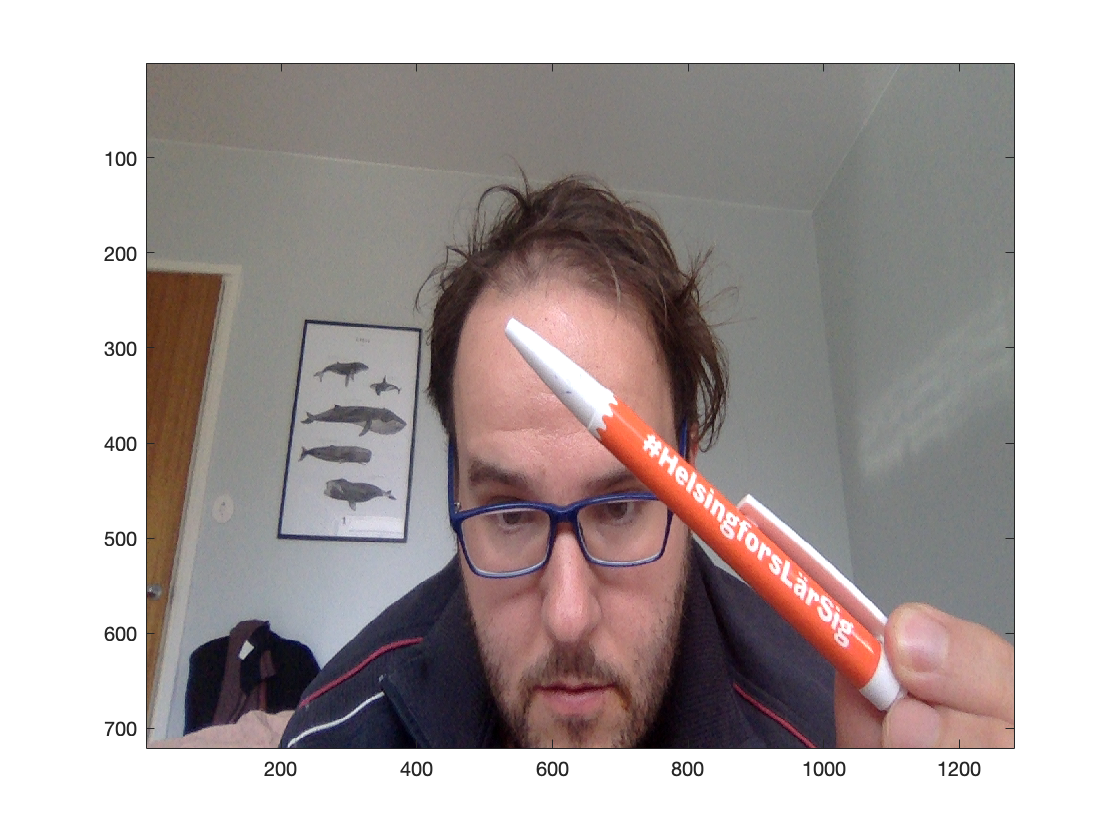}
\endminipage\hfill
\minipage{0.32\columnwidth}
  \includegraphics[width=3cm,height=2cm]{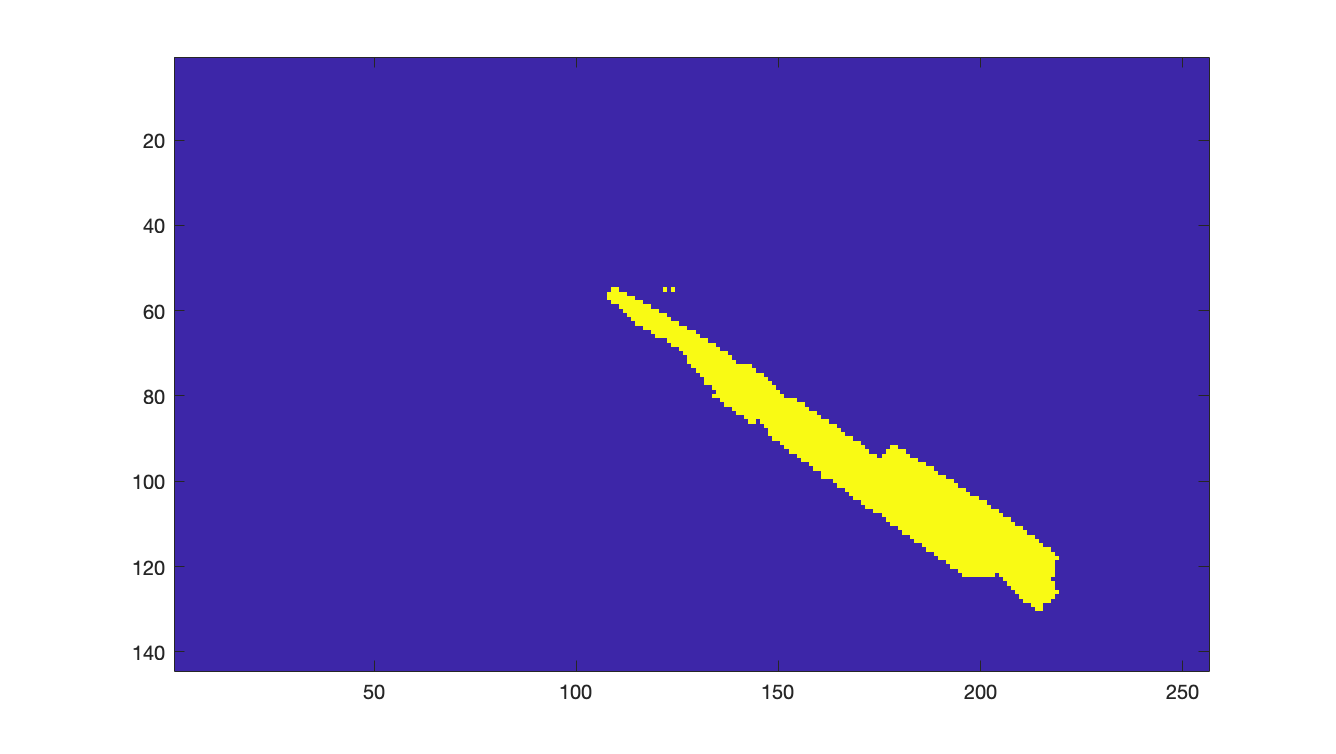}
\endminipage\hfill
\minipage{0.32\columnwidth}%
  \includegraphics[width=3cm,height=2cm]{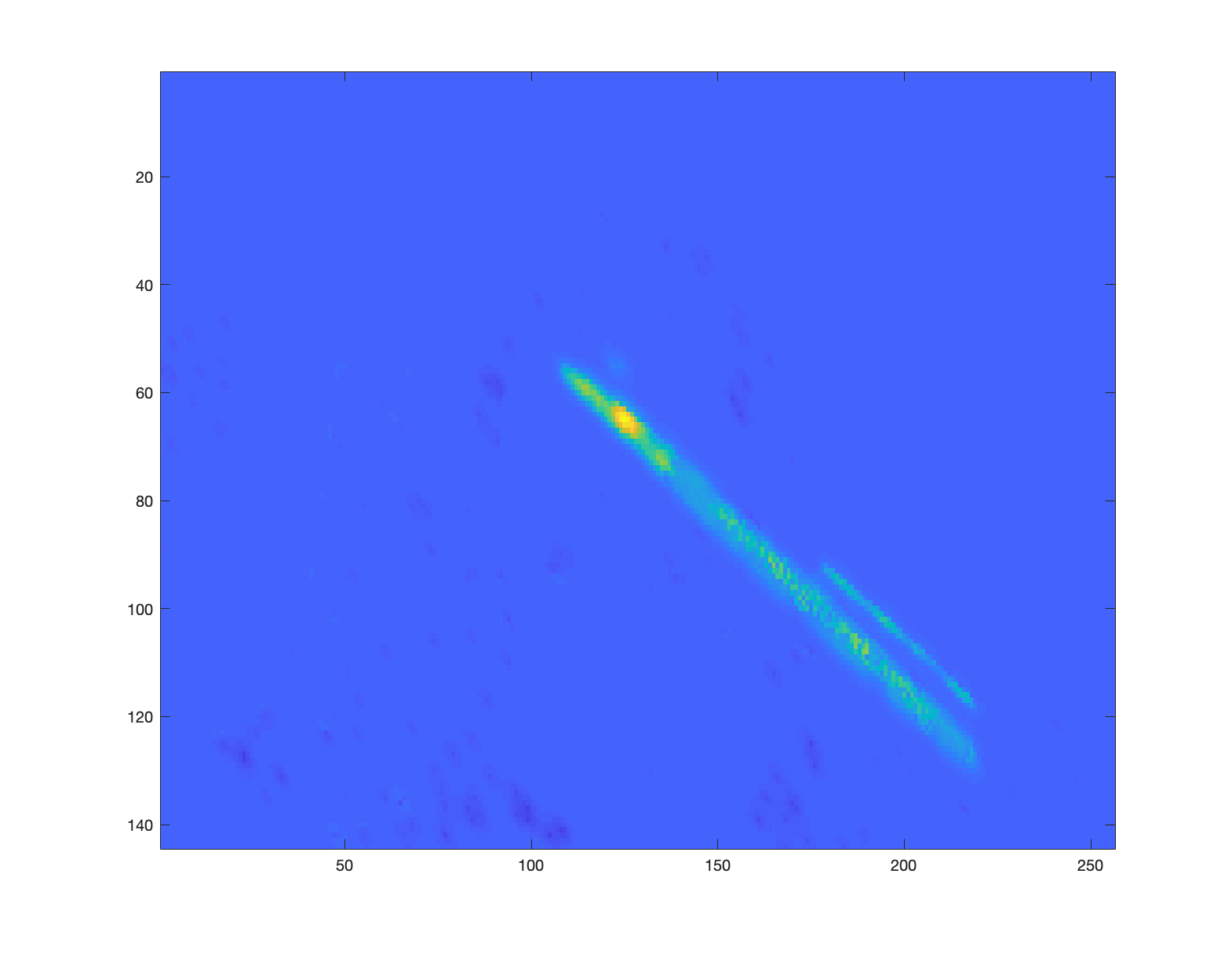}
\endminipage
  \caption{Left: Original image. Middle: Grabcut. Right: Alg.\ \ref{alg:PD}}\label{fig:awesome_image2}
\end{figure}

\vspace*{-1mm}
\section{Conclusion} 
\vspace*{-1mm}

We have introduced networked exponential families as a flexible statistical modeling paradigm for 
networked data. The error of nLasso applied to learning networked exponential families has been 
analyzed. An efficient implementation of nLasso has been proposed using a primal-dual method for 
convex optimization. Directions for future research include a more detailed analysis of the convergence 
of nLasso for typical network structures as well as data-driven learning of the network structure (graphical 
model selection). In particular, the analysis underlying Sec.\ \ref{equ_analysis_error} might guide the 
design of network structure by relating Asspt.\ \ref{def_NNSP} to network flow problems (see \cite{JungAISTATS2019}). 

\section*{Acknowledgments}
Roope Tervo from Finnish Meteorological Institute provided a Python script to download weather data.

\bibliographystyle{plain}
\bibliography{/Users/junga1/Literature.bib}

\section{Proofs} 
We first collect some helper results in Sec.\ \ref{sec_helper_results} that will 
be used in Sec.\ \ref{sec_main_proof} to obtain a detailed derivation of Theorem \ref{thm_main_result}. 
\subsection{Helper Results} 
\label{sec_helper_results}
\begin{lemma}
\label{lem_bound_noise_term}
For any two vector signals $\vu,\vv \in \mathcal{W}$ (see \eqref{equ_def_graph_sigs}) 
defined on an empirical graph $\graph$,  
\begin{align} 
\label{equ_lem_bound_noise_term_lemma}
\sum_{i \in \nodes} \big(\vu\gindex\big)^{T} \vv\gindex & \leq \nonumber \\
& \hspace*{-20mm} (1/|\nodes|) \big( \sum_{i \in \nodes} \vv\gindex \big)^{T}  \sum_{j \in \nodes} \vu^{(j)} + \big\| \big( \mD^{\dagger} \big)^{T} \vv\big\|_{2,\infty} \| \vu \|_{\rm TV}. 
\end{align}
Here, $\mD \in \mathbb{R}^{(\featurelen |\edges|) \times (\featurelen|\nodes|)}$ denotes the block-wise incidence matrix \eqref{equ_def_incident_matrix} of the empirical graph $\graph$. 
\end{lemma} 
{\bf Proof :}
Any graph signal $\vu$ can be decomposed as 
\begin{equation} 
\label{equ_decomp_proj_orth_proj}
\vu= \mathbf{P} \vu+ (\mathbf{I}- \mathbf{P}) \vu, 
\end{equation} 
with $\mathbf{P}$ denoting the orthogonal projection matrix on the nullspace of the block-wise 
graph Laplacian matrix $\mL$ \eqref{equ_def_graph_Laplacian_matrix}. 

For a connected graph, the nullspace $\mathcal{K}(\mL)$ is spanned by 
$\featuredim$ graph signals  (see \cite{Luxburg2007})
\begin{equation} 
\mathbf{v}^{(j)} = \mathbf{1} \otimes \mathbf{e}^{(j)} \in \mathcal{W} \mbox{, for } j \in \{1,\ldots,\featurelen \}.
\end{equation} 
Here, we used the constant graph signal $\mathbf{1} \in \mathbb{R}^{\nodes}$ assigning all nodes 
the same signal value $1$. The projection matrix associated with the nullspace $\mathcal{K}(\mL)$ is 
\begin{equation}
\label{equ_def_projec_connected_gaph}
\mathbf{P} = \underbrace{ (1/(\mathbf{1}^{T} \mathbf{1}))}_{= 1/|\nodes|} \sum_{j=1}^{\featurelen} \mathbf{1} \big(\mathbf{1}\big)^{T} \otimes \mathbf{M}^{(j)}. 
\end{equation} 
Here, $\mathbf{M}^{(j)} \defeq \mathbf{e}^{(j)} \big( \mathbf{e}^{(j)} \big)^{T}$. 
Therefore, 
\begin{equation} 
\label{equ_application_projection_nullspace}
\mathbf{P} \vu  \stackrel{\eqref{equ_def_projec_connected_gaph}}{=}   (1/|\nodes|)  \sum_{j=1}^{\featurelen} \sum_{i \in \nodes} u^{(i)}_{j} \mathbf{1} \otimes \mathbf{e}^{(j)}.
\end{equation}
	
The projection matrix on the orthogonal complement of  $\mathcal{K}(\mL) \subseteq \graphsigs$ is $\mathbf{I} - \mathbf{P}$. 
Then (see \cite{pmlr-v49-huetter16}), 
\begin{equation}
\label{equ_applicatio_ortho_projection}
\mathbf{I}-\mathbf{P} = \mD^{\dagger} \mD.
\end{equation}
with the block-wise incidence matrix  $\mD$ \eqref{equ_def_incident_matrix}. 
Combining \eqref{equ_application_projection_nullspace} and \eqref{equ_applicatio_ortho_projection} with \eqref{equ_decomp_proj_orth_proj}, 
\begin{align}
\label{equ_proof_basic_decomp_error}
\hspace*{-3mm}  \sum_{i \in \nodes} \big(\vu^{(i)}\big)^{T} \vv^{(i)} \!=\! 
(1/|\nodes|)\sum_{i,i' \in \nodes} \big(\vu^{(i)}\big)^{T} \vv^{(i')}\!+\!\vv^{T} \mD^{\dagger} \mD \vu.
\end{align}
Combining \eqref{equ_proof_basic_decomp_error} with the inequality  $\va^{T} \vb \leq \| \va \|_{2} \| \vb \|_{2}$,  
\begin{align}
\label{equ_proof_basic_decomp_error_12}
\sum_{i \in \nodes} \big(\vu^{(i)}\big)^{T} \vv^{(i)} & \leq \nonumber \\
& \hspace*{-20mm}  (1/|\nodes|) \sum_{i,j\in \nodes}  \big(\vu^{(i)}\big)^{T}\vv^{(j)}\!+\!\big\| \big( \mD^{\dagger} \big)^{T} \vv \big\|_{2,\infty} \| \mD \vu \|_{2,1}.
\end{align} 
The result \eqref{equ_lem_bound_noise_term_lemma} follows from \eqref{equ_proof_basic_decomp_error_12} by using \eqref{equ_TV_is_2_1_norm}. 
 \hfill$\square$

Applying Lem.\ \ref{lem_bound_noise_term} to the subgraphs induced by a 
partition $\partition\!=\!\{\cluster_{1},\ldots,\cluster_{|\partition|} \}$, yields the following result. 
\begin{corollary}
\label{cor_bound_noise_term_spectralgap}
Consider an empirical graph $\graph\!=\!(\nodes,\edges,\weightmtx)$ and partition $\partition\!=\!\{\cluster_{1},\ldots,\cluster_{|\partition|}\}$. 
Let $\cluster_{l}$ also denote the induced subgraph of a cluster and assume they are connected. 
For any two graph signals $\vu, \vv \in \graphsigs$,
\begin{align} 
\label{equ_lem_bound_noise_term}
\sum_{i \in \samplingset} \big(\vv\gindex\big)^{T} \vu\gindex & \!\leq\!  \max_{l=1,\ldots,|\partition|} (1/|\cluster_{l}|) \big\| \sum_{i \in \cluster_{l} } \vv\gindex \big\|_{2} \hspace*{-1mm}\sum_{j \in \samplingset} \| \vu^{(j)} \|_{2} \nonumber \\ 
	& \hspace*{-10mm}+ \max_{l=1,\ldots,|\partition|} \big\| \big( \mD_{\cluster_{l}}^{\dagger} \big)^{T} \vv_{\cluster_{j}} \big\|_{2,\infty} \| \vu \|_{\rm TV}. 
\end{align}
Here, $\mD_{\cluster_{l}}$ denotes the block-wise incidence 
matrix of the induced subgraph $\cluster_{l}$ (see \eqref{equ_def_incident_matrix}). 
\end{corollary}

The proof of Theorem \ref{thm_main_result} (see Section \ref{sec_main_proof}) 
will require a large deviation bound for weighted sums of independent random vectors $\vz\gindex$ 
distributed according to \eqref{equ_def_p_i}. 
\begin{lemma}
\label{lem_concentration_single} 
Consider $\samplesize$ independent random vectors $\vz\gindex$, for $i \in \samplingset$, distributed according to \eqref{equ_def_p_i}. 
For fixed unit-norm vectors $\|\mathbf{m}\gindex\|=1$, denote $y\gindex \defeq \big(\vm\gindex\big)^{T} \vt\gindex(\vz\gindex)$ 
and $\mu\gindex \defeq \expect \big\{ y\gindex \big\}$. 
If $\nabla^{2} \Phi\gindex\!\preceq\!U \mathbf{I}$ for all $i\!\in\!\samplingset$, then 
\begin{align} 
\label{equ_concentration_exp_family_bound_FIM_112}
\prob \big\{ \big| (1/\samplesize) \sum_{i \in \samplingset} \big(y\gindex\!-\!\mu\gindex\big) \big|\geq\!\eta \big\} & \!\leq\!  2 \exp\big(-\samplesize\eta^2/(2\FIMeignmax) \big).
\end{align}
\end{lemma}
\begin{proof}
Set 
\begin{equation}
\label{equ_def_sum_y_i_mu_i} 
y \defeq \sum_{i \in \samplingset} y\gindex \mbox{, and }\mu \defeq \sum_{i \in \samplingset} \mu\gindex.
\end{equation}  
By Markov's inequality, for any $\theta\!>\!0$,
\begin{align}
\label{equ_bound_upper_112}
\prob \big\{  (1/\samplesize) \sum_{i \in \samplingset} \big(y\gindex\!-\!\mu\gindex\big) \geq\!\eta \big\}  &= 
\prob \{ y - \mu \geq \samplesize \eta \} \nonumber \\[3mm]
 &\hspace*{-33mm} = \prob \{ \exp( \theta y )  \geq \exp (\theta(\samplesize\eta+\mu)) \}\nonumber \\[3mm]
&\hspace*{-33mm} \leq \exp( - \theta (\samplesize\eta + \mu ))  \expect \{ \exp ( \theta y) \} \nonumber \\[3mm]
&\hspace*{-33mm} = \exp( - \theta (\samplesize\eta + \mu ))  \prod_{i \in \samplingset} \expect \{ \exp ( \theta y\gindex) \}. 
\end{align}
The last equality in \eqref{equ_bound_upper_112} is due to the 
independence of the random variables $y\gindex$. 

Combining \eqref{equ_bound_upper_112} with  
\begin{align} 
\label{equ_moment_gen_112}
\hspace*{-2mm} \expect \{ \exp ( \theta y\gindex) \} & \stackrel{\eqref{equ_def_cummulant_function}}{=}\exp(\Phi\gindex(\trueweights\gindex\!+\!\theta\vm\gindex)\!-\!\Phi\gindex(\trueweights\gindex)) 
\end{align} 
yields 
\begin{align} 
\label{equ_upper_bound_112}
\prob \{ y\!-\!\mu\!\geq\!\eta \}& \leq  \\ 
& \hspace*{-20mm} \exp( - \theta (\samplesize \eta\!+\!\mu)\!+\!\sum_{i \in \samplingset} \Phi\gindex(\trueweights\gindex\!+\!\theta\vm\gindex)\!-\!\Phi\gindex(\trueweights\gindex)).  \nonumber
\end{align} 
Similarly, 
\begin{align} 
\label{equ_lower_bound_112}
\prob \{ y\!-\!\mu\!\leq\!- \eta \} & \!\leq\!   \\ 
& \hspace*{-20mm} \exp( - \theta (\samplesize \eta\!+\!\mu)\!+\!\sum_{i \in \samplingset}\Phi\gindex(\trueweights\gindex\!+\!\theta\vm\gindex)\!-\!\Phi\gindex(\trueweights\gindex)). \nonumber
\end{align} 
A union bound allows to sum up \eqref{equ_upper_bound_112} and \eqref{equ_lower_bound_112} to obtain 
\begin{align}
\label{equ_lower_bound_112}
\prob \{ |y\!-\!\mu|\!\geq\! \eta \} & \!\leq\!   \\ 
& \hspace*{-21mm} 2 \exp( - \theta (\samplesize \eta\!+\!\mu)\!+\!\sum_{i \in \samplingset}\hspace*{-1mm}\Phi\gindex(\trueweights\gindex\!+\!\theta\vm\gindex)\!-\!\Phi\gindex(\trueweights\gindex)). \nonumber
\end{align} 

Using Taylor's theorem and $\nabla \Phi\gindex(\trueweights\gindex) = \expect \{ \vt\gindex\}$ \cite{GraphModExpFamVarInfWainJor}, 
\begin{align} 
\Phi\gindex(\trueweights\gindex\!+\!\theta\vm\gindex)\!-\!\Phi\gindex(\trueweights\gindex) & =  \theta \mu\gindex +  \label{equ_bound_diff_cum_function_112} \\ 
& \hspace*{-40mm} (1/2) \big(\theta\gindex\big)^2 \big(\vm\gindex\big)^{T} \nabla^{2} \Phi\gindex(\trueweights\gindex\!+\!\theta\gindex \vm\gindex)\vm\gindex\nonumber
\end{align} 
 with some $\theta\gindex \!\in\![0,\theta]$. 
Inserting $\nabla^{2} \Phi\gindex \preceq \FIMeignmax \mathbf{I}$ into \eqref{equ_bound_diff_cum_function_112}, 
\begin{align} 
\Phi\gindex(\trueweights\gindex\!+\!\theta\vm\gindex)\!-\!\Phi\gindex(\trueweights\gindex) & \geq  \theta \mu\gindex + \theta^2 \FIMeignmax /2, \nonumber
\end{align} 
and, in turn via \eqref{equ_lower_bound_112}, 
\begin{align} 
\label{equ_lower_bound_1112}
\prob \{ |y\!-\!\mu|\!\geq\! \eta \} & \!\leq\!  \exp( - \theta \samplesize \eta\!+\! \samplesize \theta^{2}\FIMeignmax/2). 
\end{align} 
Optimizing \eqref{equ_lower_bound_1112} by choosing $\theta$ suitably yields \eqref{equ_concentration_exp_family_bound_FIM_112}. 
\end{proof}
Applying Lemma \ref{lem_concentration_single} using $\mathbf{m} = \mathbf{e}^{(l)}$ 
and using $\| \vx \|_{2} \leq \sqrt{\featuredim} \| \vx \|_{\infty}$, for any $\vx \in \mathbb{R}^{\featuredim}$ 
yields the following result. 
\begin{corollary}
\label{cor_concentration_single} 
Consider $\samplesize$ independent random vectors $\vz\gindex$, for $i \in \samplingset$, distributed according to \eqref{equ_def_p_i}. 
If $\nabla^{2} \Phi\gindex\!\preceq\!U \mathbf{I}$ for all $i\!\in\!\samplingset$, then 
\begin{align} 
\label{equ_concentration_exp_family_bound_FIM_1123}
\prob \big\{ \big\| (1/\samplesize) \sum_{i \in \samplingset} \big(\vz\gindex\!-\!\ \expect\{ \vz\gindex \} \big) \big\|\geq\!\eta \big\} & \!\leq\!  \nonumber \\
& \hspace*{-30mm} 2 \exp\big(-\samplesize\eta^2/(2\featuredim\FIMeignmax) \big).
\end{align}
\end{corollary}

\subsection{Proof of Theorem \ref{thm_main_result}}
\label{sec_main_proof}

Any solution $\widehat{\gsignal}$ of the nLasso problem \eqref{optProb} satisfies 
\begin{align}
\label{equ_basic_def_lasso_proof}
 & \sum_{i \in \samplingset} \hspace*{-1mm}\big[\Phi\gindex(\widehat{\vw}\gindex) \!-\! (\widehat{\gsignal}\gindex)^{T}\vt\gindex\big] \!+\!\samplesize\lambda \| \widehat{\gsignal} \|_{\rm TV} \nonumber \\
 & \leq  \hspace*{-2mm}\sum_{i \in \samplingset} \hspace*{-1mm}\big[\Phi\gindex(\trueweights\gindex) \!-\! (\trueweights\gindex)^{T}\vt\gindex\big] \!+\!\samplesize\lambda \| \trueweights \|_{\rm TV}. 
\end{align} 
We can rewrite \eqref{equ_basic_def_lasso_proof} as 
\begin{align}
\label{equ_basic_def_lasso_proof1}
 & \hspace*{-1mm} \sum_{i \in \samplingset} \hspace*{-1mm}\big({\bm \varepsilon}\gindex \big)^{T} \widehat{\vw}\gindex\!-\!\big(\bar{\vt}\gindex\big)^{T}
 \widehat{\gsignal}\gindex \!+\!\Phi\gindex(\widehat{\vw}\gindex) \!+\!\lambda\| \widehat{\gsignal}  \|_{\rm TV}  \\[2mm]
  & \hspace*{-1mm} \leq \sum_{i \in \samplingset} \hspace*{-1mm}\big({\bm \varepsilon}\gindex \big)^{T} \trueweights\gindex\!-\!\big(\bar{\vt}\gindex\big)^{T}
  \trueweights\gindex \!+\!\Phi\gindex(\trueweights\gindex) \!+\!\lambda\| \trueweights  \|_{\rm TV} \nonumber
\end{align} 
with $\bar{\vt}\gindex\!\defeq\!\expect \big\{ \vt\gindex \big\}$ and ``observation noise'' ${\bm \varepsilon}\gindex\!\defeq\!\bar{\vt}\gindex\!-\!\vt\gindex$.
To further develop \eqref{equ_basic_def_lasso_proof1}, we make use of 
\begin{equation} 
\label{equ_minimizer_Q}
\argmin_{\vw \in \mathbb{R}^{\featurelen}}\!-\!\vw^{T}\bar{\vt}\gindex\!+\!\Phi\gindex(\vw)\!=\!\trueweights\gindex, 
\end{equation} 
with the true weight vector $\trueweights\gindex$ underlying \eqref{equ_def_p_i}. 
The identity \eqref{equ_minimizer_Q} can be verified by the zero-gradient condition and evaluating the 
gradient of $\Phi\gindex(\vw)$ (see \cite[Proposition 3.1.]{GraphModExpFamVarInfWainJor}). 
Combining \eqref{equ_basic_def_lasso_proof1} with \eqref{equ_minimizer_Q}, 
\begin{align}
\label{equ_basic_def_lasso_proof2}
& \hspace*{-1mm} \sum_{i \in \samplingset} \hspace*{-1mm}\big({\bm \varepsilon}\gindex \big)^{T} \widetilde{\vw}\gindex \!+\!\lambda\| \widehat{\gsignal}  \|_{\rm TV}  \leq \lambda\| \trueweights  \|_{\rm TV} \end{align} 
with nLasso (estimation) error $\widetilde{\vw}\defeq \widehat{\vw} - \trueweights$. 

Let us assume for the moment that the observation noise ${\bm \varepsilon}\gindex$ is sufficiently small such that 
\begin{align}
\label{equ_small_noise_condition}
 \hspace*{-4mm} \big| (1/\samplesize)  \hspace*{-1mm} \sum_{i \in \samplingset} \hspace*{-2mm} \big({\bm \varepsilon}\gindex\big)^{T} 
  \widetilde{\vw}\gindex \!\big|\!\leq\!\lambda \kappa\| \widetilde{\vw} \|_{\samplingset} \!+\!(\lambda/2) \|  \widetilde{\vw} \|_{\rm TV}
\end{align} 
for every $\widetilde{\vw}\!\in\!\graphsigs$. Here, we used $\kappa \defeq \frac{K\!+\!1}{L\!-\!1}$ and 
\begin{equation} 
\nonumber
\| \vw \|_{\samplingset} \defeq \sqrt{(1/\samplesize) \sum_{i \in \samplingset}  \big\|  \vw\gindex \big\|^{2} }.
\end{equation}  

Combining \eqref{equ_small_noise_condition} with \eqref{equ_basic_def_lasso_proof2}, 
\begin{align}
 \| \widehat{\vw} \|_{\rm TV}&\!\leq\!(1/2) \| \widetilde{\vw} \|_{\rm TV} + \| \trueweights \|_{\rm TV}\!+\!\kappa \| \widetilde{\vw} \|_{\samplingset},
\end{align} 
and, in turn, via the decomposition property $ \| \vw \|_{\rm TV} = \| \vw \|_{\boundary} + \| \vw \|_{\edges \setminus \boundary}$ (see \eqref{equ_def_TV_norm_subs}), 
\begin{align}
\label{equ_proof_nLasso_is_sparse_reg}
 \| \widehat{\vw} \|_{\edges \setminus \boundary} & \!\leq\! \nonumber \\
 & \hspace*{-20mm} (1/2) \| \widetilde{\vw} \|_{\rm TV} \!+\! \| \trueweights \|_{\rm TV}\!-\!\| \widehat{\vw} \|_{\boundary}\!+\! \kappa \| \widetilde{\vw} \|_{\samplingset}  \nonumber \\
  & \hspace*{-20mm} \stackrel{(a)}{\leq} (1/2) \| \widetilde{\vw} \|_{\rm TV}\!+\! \| \trueweights \|_{\boundary}\!-\!\| \widehat{\vw} \|_{\boundary}\!+\! \kappa \| \widetilde{\vw} \|_{\samplingset} \nonumber \\
  & \hspace*{-20mm} \stackrel{(b)}{\leq} (1/2) \| \widetilde{\vw} \|_{\rm TV}\!+\! \| \trueweights \!-\!\widehat{\vw}\|_{\boundary}\!+\! \kappa \| \widetilde{\vw} \|_{\samplingset}, 
\end{align} 
where step $(a)$ is valid since we assume the true underlying weight vectors $\trueweights\gindex$ to be clustered 
according to \eqref{equ_def_clustered_signal_model}. Step $(b)$ uses the triangle 
inequality for the semi-norm $\| \cdot \|_{\boundary}$ (see \eqref{equ_def_TV_norm_subs}). 

Since $\| \widehat{\vw} \|_{\edges \setminus \boundary} = \| \widetilde{\vw} \|_{\edges \setminus \boundary}$, 
we can rewrite \eqref{equ_proof_nLasso_is_sparse_reg} as 
\begin{align}
\label{equ_proof_nLasso_is_sparse_reg1}
 (1/2) \| \widetilde{\vw} \|_{\edges \setminus \boundary} 
& \leq  (3/2) \|\widetilde{\vw}\|_{\boundary}\!+\! \kappa \| \widetilde{\vw} \|_{\samplingset} \nonumber \\
&  \stackrel{\kappa < 1}{<} (3/2) \|\widetilde{\vw}\|_{\boundary}\!+\!  \| \widetilde{\vw} \|_{\samplingset}.
\end{align}  
Thus, for sufficiently small observation noise (such that \eqref{equ_small_noise_condition} is valid), 
the nLasso error $\widetilde{\vw}\!=\!\widehat{\vw}\!-\!\trueweights$ is approximately clustered according 
to \eqref{equ_def_clustered_signal_model}. 


So far, we verified the nLasso error $\widetilde{\vw}$ to be clustered. 
For some edge $\{i,j\} \in \edges$, the error difference $\widetilde{\vw}\gindex - \widetilde{\vw}^{(j)}$, with $i,j \in \cluster_{l}$ 
belonging to the same cluster within the partition $\partition$ underlying \eqref{equ_def_clustered_signal_model}, tends to be 
small.
 
The next step is to verify that the nLasso error $\widetilde{\vw}\!=\!\widehat{\vw}\!-\!\trueweights$ (see \eqref{optProb}) 
cannot be too large. To this end, we apply the triangle inequality for TV to \eqref{equ_basic_def_lasso_proof1} yielding 
\begin{align}
\label{equ_basic_def_lasso_proof11}
 & \sum_{i \in \samplingset} \hspace*{-1mm}\big( {\bm \varepsilon}\gindex \big)^{T} \widetilde{\gsignal}\gindex\!-\!\big(\bar{\vx}\gindex\big)^{T}\widehat{\gsignal}\gindex \!+\!\Phi\gindex(\widehat{\gsignal}\gindex)  \nonumber \\[2mm]
 & \hspace*{0mm}\leq\hspace*{0mm}\sum_{i \in \samplingset} \hspace*{0mm}\!-\!\big(\bar{\vx}\gindex\big)^{T}\trueweights\gindex\!+\!\Phi\gindex(\trueweights\gindex) \!+\!\samplesize \lambda \| \widetilde{\vw}  \|_{\rm TV}. 
\end{align} 
Using Taylor's theorem and Asspt.\ \ref{asspt_FIM_lower_bound}, 
\begin{align} 
\label{equ_strong_conv_loss}
\hspace*{-4mm}\Phi\gindex(\widehat{\gsignal}\gindex)\!-\!\Phi\gindex(\trueweights\gindex)\!-\!\big(\bar{\vx}\gindex\big)^{T} \big(\widehat{\vw}\gindex\!-\!\trueweights\gindex\big) \!\geq\! \FIMeignmin \| \widetilde{\vw}\gindex \|^{2}_{2}. 
\end{align} 

Inserting \eqref{equ_strong_conv_loss} into \eqref{equ_basic_def_lasso_proof11}, 
\begin{align}
\label{equ_basic_def_lasso_proof111}
 \hspace*{-2.5mm}(1/\samplesize)\!\sum_{i \in \samplingset} \hspace*{-1mm}\big[\!-\!\big( {\bm \varepsilon}\gindex \big)^{T} \widetilde{\gsignal}\gindex \!+\! \FIMeignmin \| \widetilde{\vw}\gindex \|^{2}_{2} \big] \!\leq\!\lambda \| \widetilde{\vw} \|_{\boundary}.
\end{align} 
Combining \eqref{equ_small_noise_condition} with \eqref{equ_basic_def_lasso_proof111}, 
\begin{align}
\label{equ_basic_def_lasso_proof1112}
 \hspace*{-2.5mm}\! \FIMeignmin \| \widetilde{\vw} \|^2_{\samplingset} \!\leq\!\lambda   \| \widetilde{\vw}\|_{\boundary}\!+\!\kappa \lambda  \| \widetilde{\vw} \|_{\samplingset}.
\end{align} 
Combining \eqref{equ_proof_nLasso_is_sparse_reg1} with \eqref{equ_ineq_multcompcondition_condition} yields
\begin{align}
\label{equ_bound_compat_condition_proof}
 \| \widetilde{\vw} \|_{\boundary} \leq \kappa\| \widetilde{\vw} \|_{\samplingset}
\end{align}
and, in turn via \eqref{equ_basic_def_lasso_proof1112}, 
\begin{equation}
\label{equ_bound_lambda_kappa}
  \| \widetilde{\vw} \|_{\samplingset} \!\leq\! 2 \lambda  \kappa/\FIMeignmin.
\end{equation} 
Inserting \eqref{equ_bound_lambda_kappa} into \eqref{equ_bound_compat_condition_proof} and \eqref{equ_proof_nLasso_is_sparse_reg1}, 
\begin{align}
 \|\widetilde{\vw} \|_{\rm TV} & =  \|\widetilde{\vw} \|_{\boundary}  +\|\widetilde{\vw} \|_{\edges \setminus \boundary}  \nonumber \\ 
  & \stackrel{\eqref{equ_proof_nLasso_is_sparse_reg1}}{\leq}  \|\widetilde{\vw} \|_{\boundary}  +  3 \|\widetilde{\vw}\|_{\boundary}\!+\! \kappa \| \widetilde{\vw} \|_{\samplingset}  \nonumber \\ 
 & \stackrel{\eqref{equ_bound_compat_condition_proof}}{\leq} 5 \kappa \| \widetilde{\vw} \|_{\samplingset} \nonumber \\ 
  & \stackrel{\eqref{equ_bound_lambda_kappa}}{\leq}  5 \lambda  \kappa^2/\FIMeignmin. \label{equ_upper_boundTV_lambda}
\end{align}
According to \eqref{equ_upper_boundTV_lambda}, we can ensure a prescribed 
error level $ \|\widetilde{\vw} \|_{\rm TV} \leq \eta$ by setting ($\FIMeignmin > 1$) 
\begin{equation} 
\label{equ_choose_lambda_ensure_eta}
\lambda \defeq \eta / (5 \kappa^2).
\end{equation} 

The final step of the proof is to control the probability of \eqref{equ_small_noise_condition} 
to hold. By Cor.\ \ref{cor_bound_noise_term_spectralgap}, \eqref{equ_small_noise_condition} holds if 
\begin{align}
\label{equ_max_cluster_l_sum_noise}
\max_{\cluster_{l} \in \partition} (1/|\cluster_{l}|) \big \| \sum_{i \in \cluster_l} {\bm \varepsilon}_{i} \big\|_{2} \leq (\lambda/2) \kappa,
\end{align} 
and simultaneously 
\begin{align}
\label{equ_condition_max_infty_norm_transformed_noise}
\max_{\cluster_{l} \in \partition} \big\| \big( \mD_{\cluster_{l}}^{\dagger} \big)^{T} {\bm \varepsilon}_{\cluster_{l}} \big\|_{2,\infty} \leq M \lambda/4. 
\end{align} 
	
We first bound the probability that \eqref{equ_max_cluster_l_sum_noise} fails to hold. For a particular 
cluster $\cluster_{l}$, \eqref{equ_concentration_exp_family_bound_FIM_1123} yields 
\begin{equation}
\label{equ_bound_sum_noise_cluster}
\hspace*{-1.5mm}\prob \{  (1/|\cluster_{l}|) \big \| \sum_{i \in \cluster_l} {\bm \varepsilon}_{i} \big\|_{2} \leq (\lambda/2) \kappa\} \!\leq\! 2 \exp\bigg(\hspace*{-2mm}-\! \frac{|\cluster_{l}|\lambda^2\kappa^2}{8\featurelen \FIMeignmax} \bigg).
\end{equation}
Combining this with a union bound over all $\cluster_{l}\!\in\!\partition$ yields 
\begin{equation}
\label{equ_final_prob_bound_1}
\hspace*{-1.5mm}\prob \{\mbox{``\eqref{equ_max_cluster_l_sum_noise} invalid''} \}\!\leq\!2 |\partition| \max_{l=1,\ldots,|\partition|} \exp\bigg(\hspace*{-2mm}-\! \frac{|\cluster_{l}|\lambda^2\kappa^2}{8\featurelen \FIMeignmax} \bigg).
\end{equation}
	
For controlling the probability of \eqref{equ_condition_max_infty_norm_transformed_noise} failing to 
hold, we combine \eqref{equ_bound_cols_pseudo_D} with Lem.\ \ref{lem_concentration_single}. This yields, using a union bound 
over all edges $e \in \edges$,  
\begin{align}
\label{equ_final_prob_bound_2}
\prob \{\mbox{``\eqref{equ_condition_max_infty_norm_transformed_noise} invalid''}\}\!\leq\! 2|\edges| \exp \bigg(\!-\! \frac{\samplesize  \rho^2_{\partition} \lambda^2} {64 \FIMeignmax \featuredim \| \mathbf{A} \|^{2}_{\infty}}\bigg).
\end{align}
A union bound yields \eqref{lower_boung_prob_main_results} by summing the bounds 
\eqref{equ_final_prob_bound_1} and \eqref{equ_final_prob_bound_2} for the choice \eqref{equ_choose_lambda_ensure_eta}. 

\end{document}